\documentclass[]{article}

\usepackage[top=1in, bottom=1.25in, left=1in, right=1in]{geometry}
\usepackage[toc,page]{appendix}
\usepackage[mathscr]{euscript}
\usepackage{bm}
\usepackage{amssymb,amsmath,latexsym,amsfonts,amsthm,mathtools}
\usepackage{todonotes}

\def\hf{\hat{f}}

\def\bDelta{{\boldsymbol \Delta}}

\def\P{{\mathbb P}}
\def\bT{{\boldsymbol T}}
\def\bmu{{\boldsymbol \mu}}
\def\eps{{\varepsilon}}
\def\bZ{{\boldsymbol Z}}
\def\br{{\boldsymbol r}}
\def\ball{{\mathsf B}}
\def\loc{{\rm loc}}
\def\sc{{\rm sc}}
\def\bv{{\boldsymbol v}}
\def\bz{{\boldsymbol z}}
\def\diag{{\rm diag}}
\def\He{{\rm He}}
\def\op{\text{op}}

\def\ddiag{{\rm ddiag}}
\def\bD{{\boldsymbol D}}

\def\bw{{\boldsymbol w}}
\def\bz{{\boldsymbol z}}
\def\bs{{\boldsymbol s}}
\def\bg{{\boldsymbol g}}
\def\be{{\boldsymbol e}}
\def\bv{{\boldsymbol v}}
\def\bu{{\boldsymbol u}}
\def\ba{{\boldsymbol a}}
\def\bx{{\boldsymbol x}}
\def\bA{{\boldsymbol A}}
\def\bR{{\boldsymbol R}}
\def\bB{{\boldsymbol B}}
\def\bW{{\boldsymbol W}}
\def\bU{{\boldsymbol U}}
\def\bP{{\boldsymbol P}}
\def\bS{{\boldsymbol S}}
\def\bV{{\boldsymbol V}}
\def\bG{{\boldsymbol G}}
\def\bQ{{\boldsymbol Q}}
\def\bGamma{{\boldsymbol \Gamma}}

\def\bDelta{{\boldsymbol \Delta}}
\def\bSigma{{\boldsymbol \Sigma}}
\def\bzero{{\boldsymbol 0}}
\def\de{\mathrm{d}}
\def\tr{\text{\rm Tr}}
\renewcommand{\S}{\mathbb{S}}
\def\ones{{\boldsymbol 1}}

\def\bQ{{\boldsymbol Q}}
\def\cF{{\mathcal F}}

\def\normal{{\sf N}}
\newcommand{\E}{\mathbb{E}}
\def\NN{{\sf NN}}
\def\RF{{\sf RF}}
\def\NT{{\sf NT}}

\def\op{{\rm op}}

\def\E{{\mathbb E}}
\def\P{{\mathbb P}}
\def\bw{{\boldsymbol w}}
\def\ba{{\boldsymbol a}}
\def\bx{{\boldsymbol x}}
\def\bA{{\boldsymbol A}}
\def\bP{{\boldsymbol P}}
\def\bB{{\boldsymbol B}}
\def\bW{{\boldsymbol W}}
\def\bGamma{{\boldsymbol \Gamma}}

\def\bDelta{{\boldsymbol \Delta}}
\def\bSigma{{\boldsymbol \Sigma}}

\def\btheta{{\boldsymbol \theta}}
\def\bbeta{{\boldsymbol \beta}}
\def\bzero{{\boldsymbol 0}}
\def\de{\mathrm{d}}
\def\ba{{\boldsymbol a}}
\def\bU{{\boldsymbol U}}
\def\ones{{\mathbf 1}}

\def\bQ{{\boldsymbol Q}}
\def\cF{{\mathcal F}}

\def\normal{{\sf N}}

\def\NN{{\sf NN}}
\def\RF{{\sf RF}}
\def\NT{{\sf NT}}

\def\MG{{\sf mg}}
\def\QF{{\sf qf}}
\def\M{{\sf M}}

\def\hf{\hat{f}}

\def\sT{{\mathsf T}}

\def\Tr{\text{\rm Tr}}
\def\id{{\mathbf I}}
\newcommand{\Z}{\mathbb{Z}}
\newcommand{\R}{\mathbb{R}}
\newcommand{\<}{\langle}
\renewcommand{\>}{\rangle}

\def\prob{{\mathbb P}}
\def\reals{{\mathbb R}}
\def\eps{{\varepsilon}}
\def\bfzero{{\boldsymbol 0}}

\newtheorem{theorem}{Theorem}
\newtheorem*{theorem*}{Theorem}
\newtheorem{lemma}{Lemma}
\newtheorem{corollary}{Corollary}
\newtheorem{remark}{Remark}

\newtheorem{proposition}{Proposition}

\theoremstyle{definition}

\usepackage[utf8]{inputenc} 
\usepackage[T1]{fontenc}    
\usepackage{hyperref}       
\usepackage{url}            
\usepackage{booktabs}       
\usepackage{amsfonts}       
\usepackage{nicefrac}       
\usepackage{microtype}      

\def\sSGD{\mbox{\tiny \rm SGD}}

\def\stract{\mbox{\tiny \rm tract}}
\def\diag{{\rm diag}}
\def\He{{\rm He}}

\usepackage{hyperref}
\hypersetup{
    colorlinks,
    linkcolor={blue!80!black},
    citecolor={green!50!black},
}
\colorlet{linkequation}{blue}

\title{Limitations of Lazy Training of\\ Two-layers Neural Networks}

\author{Behrooz Ghorbani\thanks{Department of Electrical Engineering, Stanford University},\;\; Song Mei\thanks{Institute for Computational and Mathematical Engineering, Stanford University},\;\; Theodor Misiakiewicz\thanks{Department of Statistics, Stanford University}, \;\; Andrea Montanari\thanks{Department of Electrical Engineering and Department of Statistics, Stanford University}}

\begin{document}

\maketitle

\begin{abstract}
We  study the supervised learning problem under either of the following two models:
\begin{enumerate}
    \item[(1)] Feature vectors $\bx_i$ are $d$-dimensional Gaussians and responses are $y_i = f_*(\bx_i)$ for $f_*$ an unknown quadratic function;
    \item[(2)] Feature vectors $\bx_i$ are distributed as a mixture of two $d$-dimensional centered Gaussians, and $y_i$'s are the corresponding class labels. 
\end{enumerate}
We use two-layers neural networks with quadratic activations, and compare three  different learning regimes: the random features ($\RF$)
regime in which we only train the second-layer weights; the neural tangent ($\NT$) regime in which we train a linearization of the neural
network around its initialization; the fully trained neural network
($\NN$) regime in  which we train all the weights in the network. We
prove that, even for the simple quadratic model of point (1), there is
a potentially unbounded gap between the prediction risk achieved in these three training regimes, when the number of neurons is
smaller than the ambient dimension. When the number of neurons is
larger than the number of dimensions, the problem is significantly easier and both $\NT$ and $\NN$ learning achieve zero risk. 
\end{abstract}

\section{Introduction}

Consider the supervised learning problem in which we are given i.i.d. data $\{(\bx_i,y_i)\}_{i\le n}$,
where $\bx_i\sim \prob$ a probability distribution over $\reals^d$,
and $y_i = f_*(\bx_i)$. 
(For simplicity, we focus our introductory discussion on the case in
which the response $y_i$ is a noiseless function of the feature vector
$\bx_i$: some of our results go beyond this setting.)
We would like to learn the unknown function $f_*$ as to minimize the prediction risk $ \E\{(f(\bx)-f_*(\bx))^2\}$. We will assume throughout $f_*\in L^2(\R^d,\prob)$,
i.e. $\E\{f_*(\bx)^2\}<\infty$. 

The function class of two-layers neural networks (with $N$ neurons) is defined by:
\begin{align}
\cF_{\NN, N}= \Big\{ f( \bx) = c + \sum_{i=1}^N a_i \sigma(\< \bw_i,  \bx\>): \; c, a_i\in \R, \,\bw_i \in \R^d, i \in [N] \Big\}. 
\end{align}
Classical universal approximation results \cite{cybenko1989approximation} imply that any $f_*\in L^2(\reals^d,\prob)$  can be approximated arbitrarily well
by an element in $\cF_{\NN}=\cup_N \cF_{\NN, N}$ (under mild conditions). At the same time, we know that such an 
approximation can be constructed  in polynomial time only for a subset of functions $f_*$. Namely, there exist sets of
functions $f_*$ for which no algorithm can construct a good approximation in $\cF_{\NN, N}$ in polynomial time \cite{klivans2014embedding,shamir2018distribution},
even having access to the full distribution $\prob$ (under certain complexity-theoretic assumptions). 

These facts  lead to the following central question in neural network theory:
\begin{quote}
\emph{For which subset of function $\cF_{\stract}\subseteq L^2(\reals^d,\prob)$ can a neural network approximation be learnt efficiently?}
\end{quote}
Here `efficiently' can be formalized in multiple ways: in this paper we will focus on learning via stochastic gradient descent.

Significant amount of work has been devoted to two subclasses of $\cF_{\NN, N}$ which we will refer to as 
the random feature model ($\RF$) \cite{rahimi2008random}, and the neural tangent model ($\NT$) \cite{jacot2018neural}:
\begin{align}
\cF_{\RF, N}(\bW) &= \Big\{ f_N( \bx) =  \sum_{i=1}^N  a_i\sigma(\< \bw_i, \bx\>): \;  a_i \in \R, i \in [N] \Big\},\\
\cF_{\NT, N}(\bW) &= \Big\{ f_N( \bx) = c + \sum_{i=1}^N  \sigma'(\< \bw_i, \bx\>)\< \ba_i, \bx\>: c \in \R, \ba_i \in \R^d, i \in [N] \Big\}.\label{eq:NT}
\end{align}
Here $\bW=(\bw_1,\dots,\bw_N)\in \reals^{d \times N}$ are weights which are not
optimized  and instead drawn at random. Through this paper, we will assume $(\bw_i)_{i\le N}\sim_{iid}\normal(\bfzero,\bGamma)$.
(Notice that we do not add an offset in the $\RF$ model, and will limit ourselves to target functions $f_*$ that are centered: this choice simplifies some calculations
without modifying the results.)

We can think of $\RF$ and $\NT$ as \emph{tractable
inner bounds} of the class of neural networks $\NN$:\
\begin{itemize}
\item \emph{Tractable.} Both  $\cF_{\RF, N}(\bW)$, $\cF_{\NT, N}(\bW)$ are finite-dimensional linear
spaces, and minimizing the empirical risk over these classes can be performed efficiently.
\item \emph{Inner bounds.} 
Indeed $\cF_{\RF, N}(\bW) \subseteq \cF_{\NN, N}$: the random feature model is simply obtained by
fixing all the  first layer weights. Further  $\cF_{\NT}(\bW) \subseteq {\rm cl}(\cF_{\NN, 2N})$ (the closure of the class of neural networks with $2N$ neurons). This follows from
$\eps^{-1}[\sigma(\<\bw_i+\eps\ba_i,\bx\>)-\sigma(\<\bw_i,\bx\>)] = \<\ba_i,\bx\>\sigma'(\<\bw_i,\bx\>)+o(1)$ as $\eps\to 0$.
\end{itemize}
It is possible to show that the class of neural networks $\NN$ is significantly more expressive
than the two linearization $\RF$, $\NT$, see e.g. \cite{yehudai2019power,ghorbani2019linearized}.
In particular, \cite{ghorbani2019linearized} shows that, if the feature vectors $\bx_i$ are uniformly random over the $d$-dimensional sphere,
and $N,d$ are large with $N=O(d)$, then $\cF_{\RF, N}(\bW)$ can only capture linear functions, while $\cF_{\NT, N}(\bW)$
can only capture quadratic functions.

Despite these findings, it could still be that the
subset of functions $\cF_{\stract}\subseteq L^2(\reals^d,\prob)$ for which we can learn efficiently a neural network approximation is well described by $\RF$ and $\NT$. Indeed,  several recent papers 
show that  --in a certain highly overparametrized regime-- this description is accurate \cite{du2018gradient,du2018gradient2,lee2019wide}. 
A specific counterexample is given
in \cite{yehudai2019power}: if the function to be learnt is a single neuron $f_*(\bx) = \sigma(\<\bw_*,\bx\>)$ then gradient descent
(in the space of neural networks with $N=1$ neurons) efficiently learns it \cite{mei2018landscape};
on the other hand, $\RF$ or $\NT$ require a number of neurons exponential in the dimension to achieve vanishing risk.

\subsection{Summary of main results}

In this paper we explore systematically the gap between $\RF$, $\NT$ and $\NN$, by considering two specific data distributions:
\begin{enumerate}
    \item[(\QF)] Quadratic functions: feature vectors are distributed according to $\bx_i\sim \normal (\bzero , \id_d) $ and responses 
are quadratic functions $y_i = f_*(\bx_i) \equiv b_0+\< \bx_i, \bB \bx_i \>$ with $\bB \succeq 0$.
    \item[(\MG)] Mixture of Gaussians:  $y_i = \pm 1$ with equal probability $1/2$,
and $\bx_i | y_i =+1 \sim \normal ( 0 , \bSigma^{(1)})$, $\bx_i | y_i = -1 \sim \normal ( 0 , \bSigma^{(2)})$.
\end{enumerate}
Let us emphasize that the choice of quadratic functions in model \QF\, is not arbitrary: in a sense, it is the \emph{most favorable case for $\NT$  training}.
Indeed \cite{ghorbani2019linearized} proves that\footnote{Note that \cite{ghorbani2019linearized} considers feature vectors $\bx_i$ uniformly random over the sphere rather than Gaussian. However, the results of \cite{ghorbani2019linearized} can be generalized, with certain modifications, to the Gaussian case. Roughly speaking,
for Gaussian features, $\NT$ with $N=O(d)$ neurons can represent quadratic functions, and a low-dimensional subspace of higher order polynomials.} (when $N=O(d)$): 
$(i)$ Third- and higher-order polynomials cannot be approximated nontrivially by $\cF_{\NT,N}(\bW)$;
$(ii)$ Linear functions are already well approximated within $\cF_{\RF,N}(\bW)$.

For clarity, we will first summarize our result for the model \QF, and then discuss generalizations to \MG.
The prediction risk achieved within any of the regimes $\RF$, $\NT$, $\NN$  is defined by
\begin{align}
&R_{\M, N}(f_*) = \arg\min_{\hf\in\cF_{\M,N}(\bW)}\E\big\{(f_*(\bx)-\hf(\bx))^2\big\}\, ,\;\;\;\;\; \M\in\{\RF,\NT,\NN\}\,.\\
&R_{\NN, N}(f_*;\ell,\eps)  = \E\big\{(f_*(\bx)-\hf_{\sSGD}(\bx;\ell,\eps))^2\big\}\, ,
\end{align}
where $\hf_{\sSGD}(\,\cdot\,;\ell,\eps)$ is the neural network produced by $\ell$ steps of stochastic gradient descent (SGD)
where each sample is used once, and the stepsize is set to $\eps$ (see Section \ref{sec:MainNN} for a complete definition).
Notice that the quantities $R_{\M, N}(f_*)$,  $R_{\NN, N}(f_*;\ell,\eps)$ are random variables because of the random weights $\bW$,
and the additional randomness in SGD.

\begin{figure}
\phantom{A}\hspace{-0.5cm}\includegraphics[width=\textwidth]{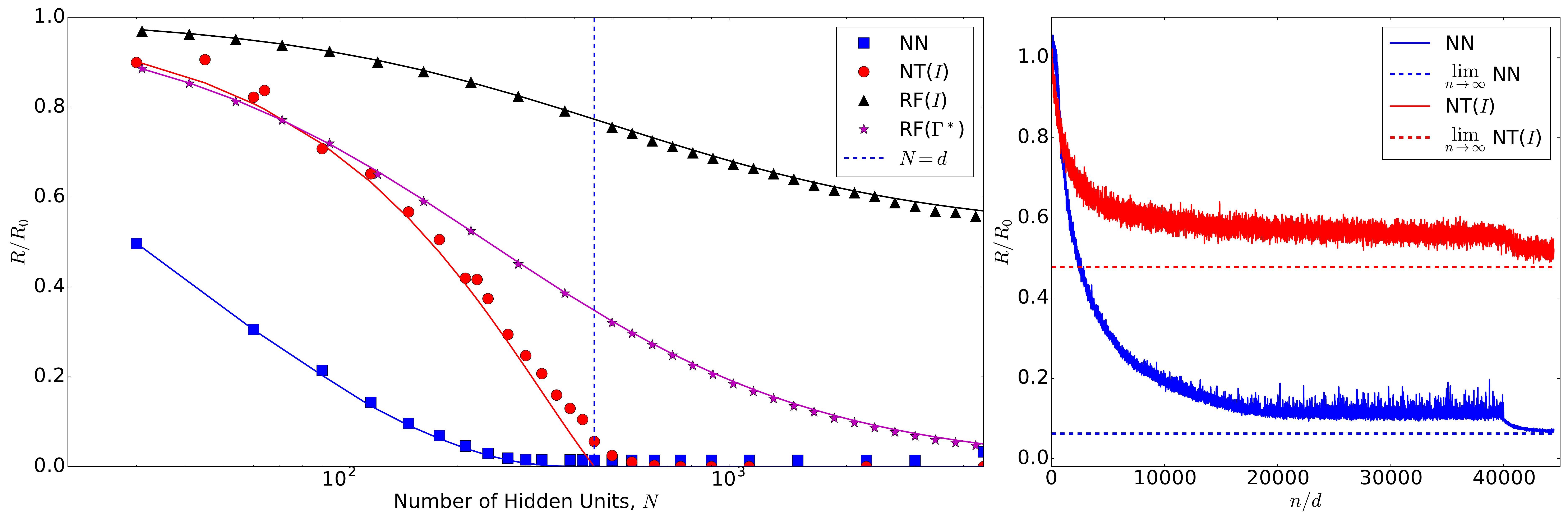}
\caption{Left frame: Prediction (test) error of a two-layer neural networks in fitting a quadratic function in $d=450$ dimensions, as a function of the number of
neurons $N$. We consider the large sample (population) limit $n\to\infty$
and compare three training regimes: random features ($\RF$), neural tangent ($\NT$), and fully trained neural
networks ($\NN$). Lines are analytical predictions obtained in this paper, and dots are empirical results. Right frame: Evolution of the risk for 
$\NT$ and $\NN$ with the number of samples. Dashed lines are our analytic prediction for the large $n$ limit.}\label{fig:Main}
\end{figure}
Our results are summarized by Figure \ref{fig:Main}, which compares the risk achieved by the three approaches
above in the population limit $n\to\infty$, using quadratic activations $\sigma(u) = u^2+c_0$.
 We consider the large-network, high-dimensional regime $N,d\to\infty$,  with $N/d\to \rho\in (0,\infty)$. 
Figure  \ref{fig:Main} reports the risk achieved by various approaches in numerical simulations, 
and compares them with our theoretical predictions for each of three regimes $\RF$, $\NT$, and $\NN$, which are detailed in the next sections.

The agreement between analytical predictions and simulations is excellent but, more
importantly, a clear picture emerges. We can highlight a few phenomena that are illustrated in this figure:

\vspace{0.05cm}

\noindent\emph{Random features} do not capture quadratic functions. The random features risk $R_{\RF, N}(f_*)$
remains generally bounded away from zero for all values of $\rho = N/d$. It is further highly dependent on the distribution 
of the weight vectors $\bw_i\sim\normal(\bfzero,\bGamma)$. Section \ref{sec:MainRF} characterizes explicitly this dependence,
for general activation functions $\sigma$.
For large $\rho= N/d$, the optimal distribution of the weight vectors uses covariance $\bGamma^* \propto \bB$, but even in this case
the risk is bounded away from zero unless $\rho\to\infty$.

\vspace{0.05cm}

\noindent\emph{The neural tangent model} achieves vanishing risk on quadratic functions for $N>d$. However,
the risk is bounded away from zero if $N/d\to \rho \in (0,1)$. Section \ref{sec:MainRF}  provides explicit expressions for the minimum risk
as a function of $\rho$. 
Roughly speaking $\NT$ fits the quadratic function $f_*$  along random subspace determined by the random weight vectors  $\bw_i$.
For $N\ge d$, these vectors span the whole space $\reals^d$ and hence the limiting risk vanishes. For $N<d$ only a fraction of the space is spanned, and not 
the most important one (i.e. not the principal eigendirections of $\bB$). 

\vspace{0.1cm}

\noindent\emph{Fully trained neural networks} achieve vanishing risk on quadratic functions for $N>d$: this is to be expected on the basis of the
previous point. For $N/d\to \rho\in (0,1)$ the risk is generally bounded away from $0$, but its value is smaller than for the neural tangent model.
Namely, in Section \ref{sec:MainNN} we give an explicit expression for the asymptotic risk (holding for $\bB\succeq \bfzero$) implying that, for some ${\rm GAP}(\rho)>0$
(independent of $N,d$),
\begin{align}
\lim_{t\to\infty}\lim_{\eps\to 0}R_{\NN, N}(f_*;\ell=t/\eps,\eps) = \inf_{f\in\cF_{\NN,N}} \E\{(f(\bx)-f_*(\bx))^2\}\le R_{\NT,N}(f_*)-{\rm GAP}(\rho)\, . \label{eq:SGD-convergence}
\end{align}
We prove this result by showing convergence of SGD to gradient flow in the population risk, and then proving a strict saddle
property for the population risk. As a consequence the limiting risk on the left-hand side coincides with the minimum risk over the whole space of neural
networks $\inf_{f\in\cF_{\NN,N}} \E\{(f(\bx)-f_*(\bx))^2\}$. We characterize the latter and shows that it amounts to fitting $f_*$ along the $N$
principal eigendirections of $\bB$.  This mechanism is very different from the one arising in the $\NT$ regime.

\vspace{0.1cm}

The picture emerging from these findings is remarkably simple. The fully trained network learns the most important eigendirections of the quadratic function
$f_*(\bx)$ and fits them, hence surpassing the $\NT$ model which is confined to a random set of directions.

Let us emphasize that the above separation between $\NT$ and $\NN$  is established only for $N\le d$.
It is natural to wonder whether this separation generalizes to $N>d$ for more complicated classes of functions, or if
instead it always vanishes for wide networks. 
We expect the separation to generalize to $N>d$ by considering higher order polynomial, instead of quadratic functions.
Partial evidence in this direction is provided by \cite{ghorbani2019linearized}: for third- or higher-order polynomials $\NT$ does not achieve vanishing risk at any $\rho\in (0,\infty)$.
The mechanism unveiled by our analysis of quadratic functions is potentially more general:
neural networks are superior to linearized models such as $\RF$ or $\NT$, because they can learn a good representation of the data.

Our results for quadratic functions are formally presented in Section \ref{sec:Quadratic}. 
In order to confirm that the picture we obtain is general, we establish similar results for mixture of Gaussians in Section \ref{sec:MainMixture}.
More precisely, our results of $\RF$ and $\NT$ for mixture of Gaussians are very similar to the quadratic case. 
In this model,  however, we do not prove a convergence result for $\NN$ analogous to \eqref{eq:SGD-convergence}, although we believe it should be possible by the same approach outlined above.
On the other hand, we characterize  the minimum prediction risk over neural networks $\inf_{f\in\cF_{\NN,N}} \E\{(y-f(\bx))^2\}$ and prove it is strictly smaller than the minimum achieved by $\RF$ and $\NT$. Finally, Section \ref{sec:Numerical} contains background on our numerical experiments.

\subsection{Further related work}

The connection (and differences) between two-layers neural networks and random features models has been the object of
several papers  since the original work of Rahimi and Recht \cite{rahimi2008random}. An incomplete list of references includes 
\cite{bach2013sharp,alaoui2015fast,bach2017breaking,bach2017equivalence,rudi2017generalization}. Our analysis contributes
to this line of work by establishing a sharp asymptotic characterization, although in more specific data distributions. 
Sharp results have recently been proven in \cite{ghorbani2019linearized}, for the special case of random weights $\bw_i$
uniformly distributed over a $d$-dimensional sphere. Here  we consider the more general case of anisotropic random 
features with covariance $\bGamma\not\propto\id$. This clarifies a key reason for  suboptimality of random features: the data representation
is not adapted to the target function $f_{*}$. 
We focus on the population limit $n\to\infty$. Complementary results characterizing the variance as a function of $n$ are given in \cite{hastie2019surprises}.

The $\NT$ model \eqref{eq:NT} is much more recent \cite{jacot2018neural}. Several papers show that SGD optimization within the original neural network is
well approximated by optimization within the model  $\NT$ as long as the number of neurons is large compared to a polynomial in the
sample size $N\gg n^{c_0}$  \cite{du2018gradient,du2018gradient2,allen2018convergence,zou2018stochastic}. 
Empirical evidence in the same direction was presented in \cite{lee2019wide,arora2019exact}. 

Chizat and Bach \cite{chizat2018note} clarified that any nonlinear statistical model can be approximated by a linear one in an
early (\emph{lazy}) training regime. The basic argument is quite simple. Given a model $\bx\mapsto f(\bx;\btheta)$ with parameters $\btheta$, we can 
Taylor-expand around a random initialization $\btheta_0$. Setting $\btheta = \btheta_0+\bbeta$, we get
\begin{align}
f(\bx;\btheta) \approx f(\bx;\btheta_0)+\bbeta^{\sT}\nabla_{\btheta}f(\bx;\btheta_0) \approx \bbeta^{\sT}\nabla_{\btheta}f(\bx;\btheta_0) \, .
\end{align}
Here the second approximation holds since, for many random initializations, $f(\bx;\btheta_0) \approx 0$ because of random cancellations.
The resulting model $\bbeta^{\sT}\nabla_{\btheta}f(\bx;\btheta_0)$ is linear, with random features.

Our objective is complementary to this literature: we prove that $\RF$ and $\NT$ have limited approximation power, and significant gain can be achieved by
full training.

Finally, our analysis of fully trained networks connects to the ample literature on non-convex statistical estimation.
For two layers neural networks with quadratic activations, Soltanolkotabi, Javanmard and Lee \cite{soltanolkotabi2019theoretical} showed that, 
as long as the number of neurons satisfies  $N\ge 2d$ there are no spurious local minimizers. Du and Lee \cite{du2018power} 
showed that the same holds as long as $N \ge d \wedge \sqrt{2n}$ where $n$ is the sample size. Zhong et. al. \cite{zhong2017recovery} 
established local convexity properties around global optima. 
Further related  landscape results include \cite{ge2017learning,haeffele2014structured,ge2017no}. 

\section{Main results: quadratic functions}
\label{sec:Quadratic}

As mentioned in the previous section, our results for quadratic functions (\QF) assume $\bx_i\sim\normal(\bfzero,\id_d)$ and
$y_i = f_*(\bx_i)$ where
\begin{align}
f_*(\bx) \equiv b_0+\< \bx , \bB \bx \>\, .\label{eq:QF}
\end{align}

\subsection{Random features}
\label{sec:MainRF}

We consider random feature model with first-layer weights $(\bw_i)_{i\le N}\sim \normal(\bfzero,\bGamma)$. We make the following assumptions:
\begin{itemize}
\item[{\bf A1.}] The activation function $\sigma$ verifies $\sigma(u)^2 \le c_0 \exp(c_1 u^2/2)$ for some constants $c_0,c_1$ with $c_1<1$.
Further it is nonlinear  (i.e. there is no $a_0, a_1\in\reals$ such that $\sigma(u) = a_0+a_1\,u$ almost everywhere).
\item[{\bf A2.}] We fix the weights' 
normalization by requiring $\E\{\|\bw_i\|^2_2\}=\Tr(\bGamma) = 1$. We assume the operator norm $\| d \cdot \bGamma \|_{\rm op} \le C$ for some constant $C$, and that the empirical spectral distribution of $d \cdot \bGamma$ converges weakly, as $d\to \infty$ to a probability distribution 
$\mathcal{D}$ over $\reals_{\ge 0}$.
 \end{itemize}
\begin{theorem}\label{thm:RF}
Let $f_*$ be a quadratic function as per Eq.~\eqref{eq:QF}, with $\E(f_*) = 0$. 
Assume conditions {\bf A1} and {\bf A2} to hold. Denote by $\lambda_k =\E_{G \sim \normal(0, 1)}[\sigma(G) \He_k(G)]$ the $k$-th Hermite coefficient of $\sigma$ and assume $\lambda_0 = 0$. Define $\tilde{\lambda} = \E_{G \sim \normal(0, 1)}[\sigma(G)^2] - \lambda^2_1$. Let $\psi>0$ be the unique solution of 
%
\begin{align}
- \tilde{\lambda}= - \frac{\rho}{\psi} + \int \frac{\lambda_1^2 t}{1 + \lambda_1^2 t \psi} \mathcal{D}(\de t)\, .
\end{align}
Then, the following holds as $N,d\to\infty$ with $N/d\to\rho$:
\begin{align}
R_{\RF,N}(f_*) = 
 \|f_{*}\|_{L_2}^2\left(1 - \frac{ \psi \lambda_2^2d \<\bGamma,\bB\>^2}{ \|\bB\|_F^2\big(2 + \psi \lambda_2^2 d\|\bGamma\|_F^2\big)}+o_{d, \P}(1) \right)\, .\label{eq:nonasy_mse}
\end{align}
 Moreover, assuming $\<\bGamma,\bB\>^2/\|\bGamma\|_F^2\|\bB\|_F^2$ to have a limit as $d\to\infty$,
\eqref{eq:nonasy_mse} simplifies as follows for   $\rho \rightarrow \infty$:
\begin{align}
\lim_{\rho \to \infty} \lim_{d \to \infty, N/d\to \rho}  \frac{R_{\RF,N}(f_*)}{\|f_{*}\|_{L_2}^2} 
&=  \lim_{d \rightarrow \infty} \left( 1- \frac{\<\bGamma,\bB\>^2}{\|\bGamma\|_F^2\|\bB\|_F^2}\right)\, .\label{eq:RF-rhoinfty}
\end{align}
\end{theorem}
Notice that  $R_{\RF,N}(f_*)/\|f_{*}\|_{L_2}^2$ is the $\RF$ risk normalized by the risk of the trivial predictor $f(\bx) = 0$. The asymptotic result in \eqref{eq:RF-rhoinfty} is remarkably simple. By Cauchy-Schwartz, the normalized risk is bounded away from zero 
even as the number of neurons per dimension diverges $\rho=N/d\to \infty$, unless $\bGamma \propto \bB$, i.e. the random features are perfectly aligned
with the function to be learned. For isotropic random features, the right-hand side of Eq.~\eqref{eq:RF-rhoinfty} reduces to 
$1-\Tr(\bB)^2/(d\|\bB\|_F^2)$. In particular, $\RF$ performs very poorly when $\Tr(\bB)\ll \sqrt{d}\|\bB\|_F$, and no better than the trivial predictor $f(\bx) =0$
if $\Tr(\bB) = 0$.

Notice that the above result applies to quite general activation functions.
The formulas simplify significantly for quadratic activations.

\begin{corollary}\label{coro:RF_quad}
Under the assumptions of Theorem \ref{thm:RF}, further assume $\sigma(x) = x^2 - 1$. Then we have, as $N,d\to\infty$ with $N/d\to\rho$:
\begin{align} \label{eq:quadratic-mse}
R_{\RF,N}(f_*)
&=  \|f_*\|_{L_2}^2\left( 1- \frac{\rho d \<\bB,\bGamma\>^2}{ \|\bB\|_F^2\big(1 + \rho d\|\bGamma\|_F^2 \big)} +o_{d, \P}(1)\right)\, .
\end{align}
\end{corollary}
The right-hand side of Eq.~\eqref{eq:quadratic-mse} is plotted in Fig.~\ref{fig:Main} for isotropic features $\bGamma= \id/d$,
and for optimal features $\bGamma=\bGamma^*\propto \bB$.

\subsection{Neural tangent}
\label{sec:MainNT}

For the $\NT$ regime, we focus on quadratic activations and isotropic weights $\bw_i \sim \normal ( \bzero, \id_d / d)$.
\begin{theorem}\label{thm:QF_NTK}
Let $f_*$ be a quadratic function as per Eq.~\eqref{eq:QF}, with $\E(f_*) = 0$, and assume $\sigma(x) = x^2$.   Then, we have for $N,d \rightarrow \infty$ with $N/d \rightarrow \rho$
\[
\E [ R_{\NT,N} (f_*)]  = \| f_* \|_{L^2}^2 \Big\{ (1-\rho)_+^2 \Big(1- \frac{\Tr(\bB)^2}{d\|\bB\|_F^2}\Big) + (1 - \rho)_+ \frac{\Tr(\bB)^2}{d\,\|\bB\|_F^2}  + o_d (1)\Big\}.
\]
where  the expectation is taken over $\bw_i \sim_{i.i.d} \normal ( \bzero, \id_d / d)$. 
\end{theorem}
As for the case of random features, the $\NT$ risk depends on the target function $f_*(\bx)$ only through the ratio $\Tr(\bB)^2/(d\,\|\bB\|_F^2)$.
However, the normalized risk is always smaller than the baseline $R_{\NT,N} (f_*)=\|f_*\|_{L^2}^2$. Note that, by Cauchy-Schwartz, 
$\E [ R_{\NT,N} (f_*)]  \le (1 - \rho)_+\|f_*\|_{L^2}^2+o_d(1)$, with this worst case achieved when $\bB\propto \id$. In particular, 
$\E [ R_{\NT,N} (f_*)]$  vanishes asymptotically  for $\rho \geq 1$. This comes at the price of a larger number of parameters
to be fitted, namely $Nd$ instead of $N$.

\subsection{Neural network}
\label{sec:MainNN}

For the analysis of SGD-trained neural networks, 
we assume  $f_*$ to be a quadratic function as per Eq.~\eqref{eq:QF}, but we will now restrict to the positive semidefinite case  $\bB \succeq 0$.
We consider quadratic activations  $\sigma(x) = x^2$, and we fix the second layers weights to be $1$:
\[
\hat f(\bx; \bW, c) = \sum_{i=1}^N \< \bw_i, \bx\>^2 + c. 
\]
Notice that we use an explicit offset to account for the mismatch in means between $f_*$ and $\hat f$.
It is useful to introduce the population risk, as a function of the network parameters $\bW, c$:
\[
L(\bW, c) = \E[(f_*(\bx) - \hat f(\bx; \bW, c))^2] =  \E\Big[\Big( \< \bx \bx^\sT, \bB - \bW \bW^\sT \> + b_0 - c \Big)^2\Big]. 
\]
Here expectation is with respect to $\bx \sim \normal(\bzero, \id_d)$.
We will study a one-pass version of SGD, whereby at each iteration $k$ we perform a stochastic gradient step with respect to a fresh sample $(\bx_k, f_*(\bx_k))$
\[
(\bW_{k + 1}, c_{k+1}) = (\bW_k, c_k) - \varepsilon \nabla_{\bW, c} \Big( f_*(\bx_k) - \hat f (\bx_k; \bW, c) \Big)^2, 
\]
and define
\[
R_{\NN, N}(f_*;\ell, \varepsilon) \equiv L(\bW_{\ell},c_{\ell}) = \E_{\bx \sim \normal(\bzero, \id_d)}[(f_*(\bx) - \hat f(\bx; \bW_\ell, c_\ell))^2 ]. 
\]
Notice that this is the risk with respect to a new sample, independent from the ones used to train $\bW_{\ell},c_{\ell}$. It is the test error.
Also notice that $\ell$ is the number of SGD steps but also (because of the one-pass assumption) the sample size.
Our next theorem characterizes the asymptotic risk achieved by SGD. This prediction is reported in Figure \ref{fig:Main}.
\begin{theorem}\label{thm:NN_quadratic}
Let $f_*$ be a quadratic function as per Eq.~\eqref{eq:QF}, with $\bB \succeq 0$. Consider SGD with initialization $(\bW_0, c_0)$ whose
distribution is absolutely continuous with respect to the Lebesgue measure.
Let $R_{\NN, N}(f_*;\ell, \varepsilon)$ be the test prediction error after $\ell$ SGD steps with step size $\eps$.

Then we have (probability is over the initialization $(\bW_0, c_0)$ and the samples)
\begin{align*}
\lim_{t\to\infty}\lim_{\varepsilon \to 0}&\P\Big( \Big \vert R_{\NN, N}(f_*;\ell=t/\varepsilon,\varepsilon) - \inf_{\bW, c} L(\bW, c) \Big\vert \ge \delta ) = 0, \\
&\inf_{\bW, c} L(\bW, c) = 2 \sum_{i= N + 1}^{d} \lambda_{i}(\bB)^2,
\end{align*}
where $\lambda_1(\bB) \ge\lambda_2(\bB)\ge \dots\ge \lambda_d(\bB)$ are the ordered eigenvalues of $\bB$. 
\end{theorem}

The proof of this theorem depends on the following proposition concerning the landscape of the population risk, which is of independent interest. 
\begin{proposition}
\label{prop:landscape_NN_QF}
Let $f_*$ be a quadratic function as per Eq.~\eqref{eq:QF}, with $\bB \succeq 0$. 
For any sub-level set of the risk function $\Omega(B_0) = \{ \bx = (\bW, c) : L(\bW, c) \le B_0 \}$, there exists constants $\varepsilon, \delta > 0$ such that $L$ 
 is $(\varepsilon, \delta)$-strict saddle in the region $\Omega(B_0)$. Namely, for any $\bx \in \Omega(B_0)$ with $\| \nabla L(\bx) \|_2 \le \varepsilon$, we have $\lambda_{\min}(\nabla^2 L(\bx)) < -\delta$. 
\end{proposition}
We can now compare the risk achieved within the regimes $\RF$,
$\NT$ and $\NN$. Gathering the results of Corollary \ref{coro:RF_quad}, and Theorems \ref{thm:QF_NTK}, \ref{thm:NN_quadratic} 
(using $\bw_i\sim\normal(0,\id/d)$ for $\RF$  and $\NT$), we obtain
\begin{align}
\frac{R_{\M,N}(f_*)}{\|f_*\|_{L_2}^2}&\approx  \begin{dcases}
1- \frac{\rho}{1+\rho} \frac{\Tr(\bB)^2}{ d\|\bB\|_F^2} & \mbox{ for $\M=\RF$,}\\
(1-\rho)_+^2+\rho(1-\rho)_+ \frac{\Tr(\bB)^2}{ d\|\bB\|_F^2} & \mbox{ for $\M=\NT$,}\\
1- \frac{\sum_{i= 1}^{d \wedge N} \lambda_{i}(\bB)^2}{\|\bB\|_F^2}  & \mbox{ for $\M=\NN$.}
\end{dcases}
\end{align}
As anticipated, $\NN$ learns the most important directions in $f_*$, while $\RF$, $\NT$ do not.

\section{Main results: mixture of Gaussians}
\label{sec:MainMixture}

\begin{figure}
\phantom{A}\hspace{-0.25cm}\includegraphics[width=1.0\textwidth]{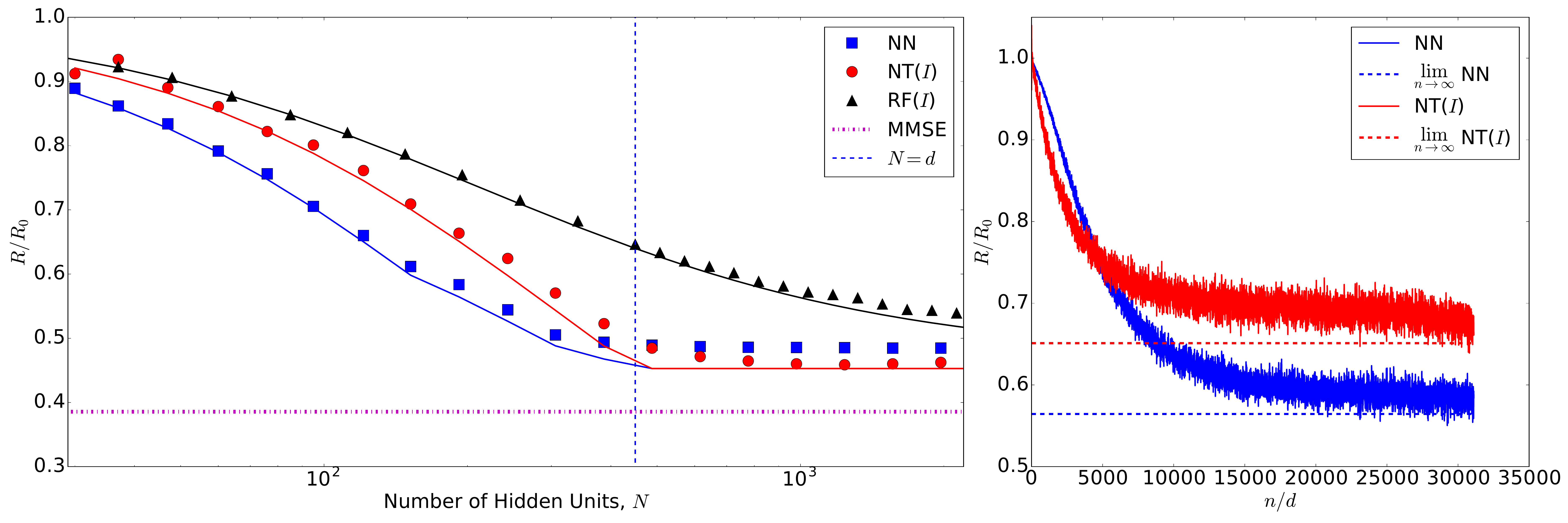}
\caption{Left frame: Prediction (test) error of a two-layer neural networks in fitting a mixture of Gaussians in $d=450$ dimensions, as a function of the number of
neurons $N$, within the three regimes $\RF$, $\NT$, $\NN$. Lines are analytical predictions obtained in this paper, and dots are empirical results (both in the population limit).
 Dotted line is the Bayes error. Right frame: Evolution of the risk for  $\NT$ and $\NN$ with the number of samples.}\label{fig:MG}
\end{figure}
In this section, we consider the mixture of Gaussian setting (\MG):  $y_i = \pm 1$ with equal probability $1/2$, and $\bx_i | y_i =+1 \sim \normal ( 0 , \bSigma^{(1)})$, $\bx_i | y_i = -1 \sim \normal ( 0 , \bSigma^{(2)} )$. We parametrize the covariances as $\bSigma^{(1)} = \bSigma - \bDelta$ and $\bSigma^{(2)} = \bSigma + \bDelta$, and 
 will make the following assumptions:
\begin{itemize}
\item[{\bf M1.}] There exists constants $0 < c_1 < c_2$ such that $c_1 \id_d \preceq \bSigma \preceq c_2 \id_d$;
\item[{\bf M2.}]  $\| \bDelta \|_{\rm op} = \Theta_d(1/\sqrt{d})$.
\end{itemize}
The scaling in assumption {\bf M2} ensures the signal-to-noise ratio to be of order one. If the eigenvalues of$\bDelta$ are much larger than $1/\sqrt{d}$, 
then it is easy to distinguish the two classes with high probability (they are asymptotically mutually singular). If   $\| \bDelta \|_{\rm op} = o_d(1/sqrt{d})$
then no non-trivial classifier exists.

We will denote by $\P_{\bSigma,\bDelta}$ the joint distribution of $(y, \bx )$ under the (\MG) model, and 
by $\E_{\bSigma,\bDelta}$ or $\E_{(y,\bx)}$ the corresponding expectation.
The minimum prediction risk within any of the regimes $\RF$, $\NT$, $\NN$ is defined by 
\[
R_{\M,N}(\P) = \inf_{f\in\cF_{\M,N}}\E_{(y, \bx)}\{(y-f(\bx))^2\}\, ,\;\;\;\;\; \M\in\{\RF,\NT,\NN\}\,.
\]

As mentioned in the introduction, the picture emerging from our analysis of the \MG\, model is aligned with the results
obtained in the previous section. We will  limit ourselves to stating the results without repeating comments that were made above.
Our results are compared with simulations in Figure \ref{fig:MG}. Notice that, in this case, the Bayes error (MMSE) is not achieved even for very wide networks
$N/d\gg 1$ either by $\NT$ or $\NN$. 

\subsection{Random seatures}
\label{sec:RFMixture}

As in the previous section, we generate random first-layer weights $(\bw_i)_{i\le N}\sim \normal(\bfzero,\bGamma)$. We consider a general activation function satisfying condition ${\bf A1}$.
We make the following assumption on $\bGamma,\bSigma$:
\begin{itemize}
\item[{\bf B2.}] We fix the weights' 
normalization by requiring $\E\{\<\bw_i,\bSigma\bw_i\>\}=\Tr(\bGamma \bSigma) = 1$. We assume that there exists a constant $C$ such that $\| d \cdot \bGamma \|_{\rm op} \le C$,
and that the empirical spectral distribution of 
$d \cdot (\bGamma^{1/2} \bSigma \bGamma^{1/2})$ converges weakly, as $d\to \infty$ to a probability distribution $\mathcal{D}$ over $\reals_{\ge 0}$.
\end{itemize}

\begin{theorem}\label{thm:RFMixture}
Consider the \MG\, distribution, with $\bSigma$ and $\bDelta$ satisfying condition {\bf M1} and {\bf M2}. 
Assume conditions {\bf A1}  and {\bf B2} to hold.   Define $\lambda_k =\E_{G \sim \normal(0, 1)}[\sigma(G) \He_k(G)]$ to be the $k$-th Hermite coefficient of $\sigma$ and assume without loss of generality $\lambda_0 = 0$. Define $\tilde{\lambda} = \E[\sigma(G)^2] - \lambda^2_1$.  Let $\psi>0$ be the unique solution of
\begin{align}
- \tilde{\lambda} = - \frac{\rho}{\psi} + \int \frac{\lambda_1^2 t}{1 + \lambda_1^2 t \psi} \,\mathcal{D}(\de t)\, .
\end{align}
Define $\zeta_1(d)\equiv d\, \Tr (\bSigma\bGamma \bSigma \bGamma) / 2$,
$\zeta_2(d) \equiv d\, \Tr (\bDelta \bGamma)^2/4$. 
Then, the following holds as $N,d\to\infty$ with $N/d\to\rho$:
\begin{align}\label{eq:nonasy_mse_mixture}
R_{\RF,N}(\P_{\bSigma, \bDelta}) &= 
\frac{1 + \zeta_1(d) \lambda_2^2 \psi}{1 + (\zeta_1(d) + \zeta_2(d)) \lambda_2^2 \psi } + o_{d, \P}(1),\, .
\end{align}
Moreover, assume $\zeta_1(d)$ $\zeta_2(d)$ to have limits as $d\to\infty$, i.e. we have $ \lim_{d \to \infty} \zeta_j (d) = \zeta_{j,*} $ for $j= 1,2$. Then the following holds as $\rho \rightarrow \infty$:
\begin{align}
\lim_{\rho \to \infty} \lim_{d \to \infty, N/d\to \rho} R_{\RF,N}(\P_{\bSigma, \bDelta})
= \frac{\zeta_{1,*} }{ \zeta_{1,*} + \zeta_{2,*} }.
\end{align}
\end{theorem}

\subsection{Neural tangent}
\label{sec:NTMixture}

For the $\NT$ model, we first state our theorem for general $\bSigma$ and $\bw_i \sim \normal ( \bzero , \bGamma )$ and then give an explicit concentration result in the case $\bSigma = \id$ and isotropic weights $\bw_i \sim \normal ( \bzero , \id / d )$. 
\begin{theorem}
\label{thm:NTK_MG}
Let $\P_{\bSigma,\bDelta}$ be the mixture of Gaussian distribution, with $\bSigma$ and $\bDelta$ satisfying conditions {\bf M1} and {\bf M2}. Further assume $\sigma(x) = x^2$. Then, the following holds for almost every $\bW \in \R^{d \times N}$ (with respect to the Lebesgue measure): 
\[
R_{\NT,N} (\P_{\bSigma,\bDelta}) = \frac{2}{2 + \| \Tilde \bDelta \|_F^2 - \| \bP_{\perp } \Tilde \bDelta \bP_{\perp } \|^2_F }+ o_{d} (1),
\]
where $\Tilde \bDelta = \bSigma^{-1/2} \bDelta \bSigma^{-1/2}$ and $\bP_{\perp} = \id - \bSigma^{1/2} \bW ( \bW^\sT \bSigma \bW)^{-1} \bW^\sT \bSigma^{1/2}$ is the projection perpendicular to $\text{{\rm span}} ( \bSigma^{1/2} \bW)$. 

Assuming further that $\bSigma = \id$ and $\bw_i \sim_{i.i.d.} \normal ( \bzero , \id_d / d )$, we have as $N,d \to \infty$ with $N/d \to \rho$:
\begin{align*}
&R_{\NT,N} (\P_{\id,\bDelta})  = \frac{2}{2 +  \kappa (\rho , \bDelta) \, \| \bDelta \|^2_F }+ o_{d,\P} (1),\\
& \kappa (\rho , \bDelta ) = 1 - (1 - \rho)_+^2 \Big(1 -  \frac{\Tr (\bDelta)^2}{d \| \bDelta \|_F^2 } \Big) - (1 - \rho)_+  \frac{\Tr (\bDelta)^2}{d \| \bDelta \|_F^2 }  ,
\end{align*}
In particular, for $\rho \geq 1$, we have (for almost every $\bW$)
\[
R_{\NT,N} (\P_{\id,\bDelta})    = \frac{1}{1 + \| \bDelta \|^2_F/2} + o_{d,\P} (1).
\]
\end{theorem}

\subsection{Neural network}
\label{sec:NNMixture}

We consider quadratic activations with general offset and coefficients
$\hat f (\bx ; \bW, \ba, c) = \sum_{i=1}^N a_i \< \bw_i , \bx \>^2 +c$.
This is optimized over $(a_i,\bw_i)_{i\le N}$ and $c$.
\begin{theorem}\label{thm:NN_MG}
Let $\P_{\bSigma,\bDelta}$ be the mixture of Gaussian distribution, with $\bSigma$ and $\bDelta$ satisfying conditions {\bf M1} and {\bf M2}. Then, the following holds
\[
R_{\NN,N} (\P_{\bSigma,\bDelta}) = \frac{2}{2 + \sum_{i= 1}^{N\wedge d} \lambda_{i}(\Tilde \bDelta)^2}+ o_d (1),
\]
where $\Tilde \bDelta = \bSigma^{-1/2} \bDelta \bSigma^{-1/2}$ and $\lambda_1 (\Tilde \bDelta) \ge \lambda_1 (\Tilde \bDelta) \ge\dots\ge \lambda_d(\Tilde\bDelta)$
 are the singular values of $\Tilde \bDelta$.
In particular, for $\rho \geq 1$, we have
\[
R_{\NN, N} (\P_{\id,\bDelta})    = \frac{1}{1 + \| \Tilde\bDelta \|^2_F/2}+ o_{d} (1).
\]
\end{theorem}
Let us emphasize that, for this setting,  we do not have a convergence result for SGD as for the model \QF, cf. Theorem \ref{thm:NN_quadratic}. 
However, because of certain analogies between the two models, we expect a similar result to hold for mixtures of Gaussians.

We can now compare the risks achieved within the regimes $\RF$,
$\NT$ and $\NN$. Gathering the results of Theorems \ref{thm:RFMixture}, \ref{thm:NTK_MG} and \ref{thm:NN_MG} for $\bSigma = \id$ and $\sigma(x) = x^2 -1$ (using $\bw_i \sim \normal ( \bzero , \id/d)$ for \RF\, and \NT), we obtain
\begin{align}
R_{\M,N}(\P_{\id,\bDelta}) &\approx  \begin{dcases}
 \frac{1 }{1+  \frac{\rho}{1 + 2\rho} \cdot \frac{\tr (\bDelta)^2}{2 d  }}  & \mbox{ for $\M=\RF$,}\\
 \frac{1}{1 + \kappa ( \rho , \bDelta ) \| \bDelta \|_F^2 /2}  & \mbox{ for $\M=\NT$,}\\
\frac{1}{1 + \sum_{i= 1}^{N\wedge d} \lambda_{i}(\bDelta)^2/2}  & \mbox{ for $\M=\NN$.}
\end{dcases}
\end{align}
We recover a similar behavior as in the case of the (\QF ) model: $\NN$ learns the most important directions of $\bDelta$, while $\RF$, $\NT$ do not. Note that the Bayes error is not achieved in this model.

\section{Numerical Experiments}
\label{sec:Numerical}
For the experiments illustrated in Figures \ref{fig:Main} and \ref{fig:MG}, we use feature size of $d=450$, and  number of hidden units 
$N \in \{45, \cdots, 4500\}$. $\NT$ and $\NN$ models are trained with SGD in TensorFlow \cite{abadi2016tensorflow}. We run a total of $2.0 \times 10^5$ SGD steps for each (\QF) model and $1.4 \times 10^5$ steps for each (\MG) model. The SGD batch size is fixed at $100$ and the step size is chosen from the grid $\{0.001,\cdots, 0.03\}$ where the hyper-parameter that achieves the best fit is used for the figures. $\RF$ models are fitted directly by solving KKT conditions with $5.0 \times 10^{5}$ observations. After fitting the model, the test error is evaluated on $1.0 \times 10^4$ fresh samples. In our figures, each $\RF$ data point corresponds to the test error averaged over $10$ models with independent realizations of $\bW$. 

For (\QF)  experiments, we choose 
$\bB$ to be diagonal with diagonal elements chosen i.i.d from standard exponential distribution with parameter $1$. For (\MG) experiments, $\bDelta$ is also diagonal with the diagonal 
element chosen uniformly from the set $\{ \frac{2}{\sqrt{d}}, \frac{1.5}{\sqrt{d}}, \frac{1}{\sqrt{d}}\}$.

\section*{Acknowledgements}

This work was partially supported by grants NSF DMS-1613091, CCF-1714305, IIS-1741162, and
ONR N00014-18-1-2729, NSF DMS-1418362, NSF DMS-1407813.

\bibliographystyle{amsalpha}

\newcommand{\etalchar}[1]{$^{#1}$}
\providecommand{\bysame}{\leavevmode\hbox to3em{\hrulefill}\thinspace}
\providecommand{\MR}{\relax\ifhmode\unskip\space\fi MR }
\providecommand{\MRhref}[2]{%
  \href{http://www.ams.org/mathscinet-getitem?mr=#1}{#2}
}
\providecommand{\href}[2]{#2}

\clearpage

\appendix

\section{Technical background}

\subsection{Hermite polynomials}

The Hermite polynomials $\{\He_k\}_{k\ge 0}$ form an orthogonal basis of $L^2(\reals,\gamma)$, where $\gamma(\de x) = e^{-x^2/2}\de x/\sqrt{2\pi}$ 
is the standard Gaussian measure, and $\He_k$ has degree $k$. We will follow the classical normalization (here and below, expectation is with respect to
$G\sim\normal(0,1)$):
\begin{align}
\E\big\{\He_j(G) \,\He_k(G)\big\} = k!\, \delta_{jk}\, .
\end{align}
As a consequence, for any function $g\in L^2(\reals,\gamma)$, we have the decomposition
\begin{align}
g(x) = \sum_{k=0}^{\infty}\frac{\mu_k(g)}{k!}\, \He_k(x)\, ,\;\;\;\;\;\; \mu_k(g) \equiv \E\big\{g(G)\, \He_k(G)\}\, .
\end{align}

\subsection{Notations}
Throughout the proofs, $O_d(\, \cdot \, )$  (resp. $o_d (\, \cdot \,)$) denotes the standard big-O (resp. little-o) notation, where the subscript $d$ emphasizes the asymptotic variable. We denote $O_{d,\P} (\, \cdot \,)$ (resp. $o_{d,\P} (\, \cdot \,)$) the big-O (resp. little-o) in probability notation: $h_1 (d) = O_{d,\P} ( h_2(d) )$ if for any $\eps > 0$, there exists $C_\eps > 0 $ and $d_\eps \in \Z_{>0}$, such that
\[
\begin{aligned}
\P ( |h_1 (d) / h_2 (d) | > C_{\eps}  ) \le \eps, \qquad \forall d \ge d_{\eps},
\end{aligned}
\]
and respectively: $h_1 (d) = o_{d,\P} ( h_2(d) )$, if $h_1 (d) / h_2 (d)$ converges to $0$ in probability.

We will occasionally hide logarithmic factors  using the  $\Tilde O_d (\, \cdot\, )$ notation (resp. $\Tilde o_d (\, \cdot \, )$): $h_1(d) = \tilde O_d(h_2(d))$ if there exists a constant $C$ 
such that $h_1(d) \le C(\log d)^C h_2(d)$. Similarly, we will denote $\Tilde O_{d,\P} (\, \cdot\, )$ (resp. $\Tilde o_{d,\P} (\, \cdot \, )$) when considering the big-O in probability notation up to a logarithmic factor.

\section{Proofs for quadratic functions}

Our results for quadratic functions (\QF) assume $\bx_i\sim\normal(0,\id_d)$ and $y_i = f_*(\bx_i)$ where
\begin{align}
f_{*}(\bx_i) \equiv b_0+\< \bx , \bB \bx \>\, .\label{eq:QF_proof}
\end{align}
Throughout this section, we will denote $\E_{\bx}$ the expectation operator with respect to $\bx \sim\normal(0,\id_d)$, and $\E_{\bw}$ the expectation operator with respect to $\bw \sim\normal(0,\bGamma)$.

\subsection{Random Features model: proof of Theorem \ref{thm:RF}}
\label{sec:proof_MainRF}

Recall the definition
\[
R_{\RF, N}(f_*) = \arg\min_{\hf\in\cF_{\RF,N}(\bW)}\E\big\{(f_* (\bx) -\hf(\bx))^2\big\},
\]
where
\[
\cF_{\RF, N}(\bW) = \Big\{ f_N( \bx) =  \sum_{i=1}^N  a_i\sigma(\< \bw_i, \bx\>): \; a_i \in \R, i \in [N] \Big\}.
\]
Note that it is easy to see from the proof that the result stays the same if we add an offset $c$.

\subsubsection{Representation of the $\RF$ risk}
\begin{lemma}\label{lem:QF_risk_representation}
Consider the $\RF$ model. We have
\begin{align}
R_{\RF,N}(f_*) = \E_{\bx }[f_*(\bx) ^2] - \bV^\sT \bU^{-1} \bV, 
\end{align}
where $\bV = [V_1, \ldots, V_N]^\sT$, and $\bU = (U_{ij})_{i, j \in [N]}$, with
\[
\begin{aligned}
V_i =& \E_\bx[f_*(\bx) \sigma(\< \bw_i, \bx\>)], \\
U_{ij} =& \E_\bx[\sigma(\< \bw_i, \bx\>) \sigma(\< \bw_j, \bx\>)].
\end{aligned}
\]
\end{lemma}
\begin{proof}[Proof of Lemma \ref{lem:QF_risk_representation}. ]
Simply write the KKT conditions. The optimum is achieved at $\ba = \bU^{-1} \bV$. 
\end{proof}

\subsubsection{Approximation of kernel matrix $\bU$}

\begin{lemma}\label{lem:RF_QF_kernel}
Let $\sigma \in L^2(\R ,\gamma)$ be an activation function. Denote $\lambda_k =\E_{G \sim \normal(0, 1)}[\sigma(G) \He_k(G)]$ the $k$-th Hermite coefficient of $\sigma$ and assume $\lambda_0 = 0$. 
Let $\bU = (U_{ij})_{i, j \in [N]}$ be a random matrix with
\[
\begin{aligned}
U_{ij} =& \E_{\bx}[\sigma(\< \bw_i, \bx\>) \sigma(\< \bw_j, \bx\>)],
\end{aligned}
\]
where $(\bw_i)_{i \in [N]} \sim \normal(\bzero, \bGamma)$ independently. Assume conditions {\bf A1} and {\bf A2} hold. 

Let $\bW = (\bw_1, \ldots, \bw_N) \in \R^{d \times N}$, and denote $\bU_0 = \{(U_0)_{ij}\}_{i,j\in[N]}$, with
\[
(U_0)_{ij} = \tilde \lambda \delta_{ij} + \lambda_1^2 \< \bw_i, \bw_j\> + \kappa / d + \mu_i \mu_j,
\]
where 
\[
\begin{aligned}
\mu_i =& \lambda_2 (\| \bw_i \|_2^2 - 1) / 2,\\
\tilde \lambda =& \E[\sigma(G)^2] - \lambda^2_1, \\
\kappa = & d \lambda_2^2 \tr(\bGamma^2)/2.
\end{aligned}
\]
Then we have as $N/d = \rho$ and $d \to \infty$, 
\[
\| \bU - \bU_0 \|_{\op} = o_{d, \P}(1).
\]

\end{lemma}

\begin{proof}[Proof of Lemma \ref{lem:RF_QF_kernel}]~

\noindent
{\bf Step 1. Hermite expansion of $\sigma$ for $\| \bw_i \|_2 \neq 1$. } Denote $\sigma_i (x) = \sigma(\| \bw_i \|_2 \cdot x)$. First notice that by a change of variables, we get
\begin{equation}
\E[ \sigma(tG) ] = \E [ (\sigma(G) / t) \exp( G^2 (1 - 1/t^2)/2) ].
\label{eq:change_var_hermite}
\end{equation}
By Assumption {\bf A1}, there exists $c_1 <1$ such that
\[
\sigma(u)^2 \exp( u^2 (1 - 1/t^2)) \leq c_0 \exp ( u^2 ( c_1/2 + 1- 1/t^2) ).
\]
Hence for $|t-1|$ sufficiently small, we have $\sigma_i \in L^2(\R , \gamma)$ and we can consider its Hermite expansion
\[
\sigma_i (x) = \sum_{k = 0}^\infty \frac{\zeta_k(\sigma_i)}{k!} \He_k(x),
\]
where
\[
\zeta_k(\sigma_i) = \E_{G \sim \normal(0, 1)}[\sigma(\| \bw_i \|_2 G) \He_k(G)].
\]
Denote the Hermite expansion of $\sigma$ to be 
\[
\sigma(x) = \sum_{k = 0}^\infty \lambda_k(\sigma) \He_k(x)/k!,
\]
where
\[
\lambda_k(\sigma) = \E_{G \sim \normal(0, 1)}[\sigma( G) \He_k(G)].
\]
By dominated convergence theorem, we have
\[
\lim_{t \to 1}\E_{G \sim \normal(0, 1)}[(\sigma(G) - \sigma(tG))^2] = 0.
\]
In addition, by sub-Gaussianity of the norm of a multivariate Gaussian random variable (see \cite{vershynin2010introduction}), it is easy to show that
\begin{equation}
\sup_{i \in [N]} \vert \| \bw_i \|_2 - 1 \vert = o_{d, \P}(1).
\label{eq:sup_wi_RF_QF}
\end{equation}
Hence we have
\begin{align}
\sup_{i \in [N]} \| \sigma - \sigma_i \|_{L^2} =& o_{d, \P}(1), \nonumber \\
\sup_{i \in [N]} \vert \zeta_k(\sigma_i) - \lambda_k(\sigma) \vert \le& \sup_{i \in [N]} \| \sigma - \sigma_i \|_{L^2} \E[\He_k(G)^2]^{1/2} =  o_{d, \P}(1), \label{eq:control_sigma_i}
\end{align}
for any fixed integer $k$. 

\noindent
{\bf Step 2. Expansion of $\bU$. } Denote $\bu_i = \bw_i / \| \bw_i \|_2$, then we have
\begin{equation}
U_{ij} = \underbrace{\zeta_0(\sigma_i) \zeta_0(\sigma_j)}_{T_{0, ij}} + \underbrace{\zeta_1(\sigma_i) \zeta_1(\sigma_j)\< \bu_i, \bu_j\>}_{T_{1, ij}} + \underbrace{\zeta_2(\sigma_i) \zeta_2(\sigma_j) \frac{\< \bu_i, \bu_j\>^2}{2}}_{T_{2, ij}} +  \underbrace{\sum_{k\geq 3} \zeta_k(\sigma_i) \zeta_k(\sigma_j) \frac{\< \bu_i, \bu_j\>^k}{ k!}}_{T_{3, ij}}. 
\label{eq:decomposition_U_QF}
\end{equation}
We define
\[
\bT_k = \left(\zeta_k(\sigma_i) \zeta_k(\sigma_j)\frac{ \< \bw_i, \bw_j\>^k}{k! } \right)_{i, j \in [N]}. 
\]

\noindent
{\bf Step 3. Term $\bT_0$. } By definition of $\mu_i$, we have
\[
\bT_0 = (\zeta_0(\sigma_i) \zeta_0(\sigma_j))_{i, j \in [N]} = \bD_{0} [ (\lambda_2/2)^2 (\| \bw_i \|_2^2 - 1) (\| \bw_j \|_2^2 - 1)]_{i, j \in [N]} \bD_{0},
\]
where (by the assumption that $\E_G[\sigma(G)] = 0$)
\[
(\bD_0)_{ii} =  \frac{\zeta_0(\sigma_i) }{   \lambda_2 (\| \bw_i \|_2^2 - 1)/2}= \E \Big[ \frac{\sigma(\| \bw_i \| G) - \sigma(G)}{  \| \bw_i \|_2 - 1}\Big] \cdot \frac{1}{ \lambda_2 (\| \bw_i \|_2 + 1)/2}.
\]
Let us show:
\begin{equation}
\lim_{t \to 1} \E\Big[\frac{\sigma(tG) - \sigma(G) }{t - 1} \Big] = \lambda_2(\sigma), 
\label{eq:conv_lambda_2_dom}
\end{equation}
or equivalently:
\[
\lim_{t \to 1} \E \Big[ \frac{ \sigma(tG) - \sigma(G) }{t - 1} - (G^2 - 1) \sigma(G) \Big] = 0
\]
Recall the change of variable \eqref{eq:change_var_hermite} and do a first order Taylor expansion of the exponential: there exists a function $\xi ( G) \in [0,G]$ such that
\[
\begin{aligned}
& \E \Big[ \frac{ \sigma(tG) - \sigma(G) }{t - 1} - (G^2 - 1) \sigma(G) \Big] \\
= &\E \Big[ \sigma(G) \Big( \exp( G^2 (1 - 1/t^2)/2)  - t - t(t-1) (G^2 - 1) \Big) \Big] \cdot \frac{1}{t(t-1)} \\
= & \E \Big[ \sigma(G) (t-1) \Big(  1 - G^2 [2t + 1 ] /(2t^2)  +  G^4 (t+1)^2/(8t^4) \exp ( \xi(G)^2 (1-1/t^2)/2) \Big) \Big] \cdot \frac{1}{t} .
\end{aligned}
\]
We see that the integrand goes to zero as $t \to 1$.  For $|t - 1|$ sufficiently small, we have
\[
\frac{\Big\vert \exp( G^2 (1 - 1/t^2)/2)  - t - t(t-1) (G^2 - 1) \Big\vert }{| t-1|} \leq 2 + 2 G^2 + 2 G^4  \exp( G^2/5 ),
\]
which is squared integrable. Recalling that $\sigma \in  L^2(\R , \gamma)$, we obtain \eqref{eq:conv_lambda_2_dom} by dominated convergence.

Hence, combining \eqref{eq:sup_wi_RF_QF} and \eqref{eq:conv_lambda_2_dom} gives 
\[
\| \bD_0 - \id_d \|_{\op} = o_{d, \P}(1). 
\]
Furthermore, for $\bmu = (\mu_i)_{i\in[N]}$ with $\mu_i = \lambda_2 (\| \bw_i \|_2^2 - 1)/2$, we have
\[
\E [ \| \bmu \bmu^\sT \|_{\op} ] = \E [ \| \bmu \|_2^2] = \frac{\lambda_2^2}{4}	N \E [ (\| \bw_i \|^2_2 - 1)^2 ] = \frac{\lambda_2^2}{2}	N \| \bGamma \|_F^2 \leq \frac{\lambda_2^2}{2}	N^2 \| \bGamma \|_{\op}^2 = O_{d,\P} (1),
\]
where the last equality comes from assumption {\bf A2}.  We  get
\begin{equation}
  \| \bT_0 - \bmu \bmu^\sT \|_{\op} \leq 2 \| \bD_0 - \id_d \|_{\op}  \| \bmu \bmu^\sT \|_{\op}  ( \| \bD_0 \|_{\op} +1) = o_{d, \P}(1).
  \label{eq:U_bound_1_QF}
\end{equation}

\noindent
{\bf Step 4. Term $\bT_1$. }
For $\bT_1$, we have
\[
\bT_1 = (\zeta_1(\sigma_i) \zeta_1(\sigma_j)\< \bu_i, \bu_j\>)_{i, j \in [N]} = \bD_1 \bW^\sT \bW \bD_1,
\]
where 
\[
\bD_1 = \diag( (\zeta_1(\sigma_i)) / \| \bw_i \|_2 ). 
\]
By the uniform convergence of $\zeta_1(\sigma_i)$ to $\lambda_1(\sigma)$, cf Eq.~\eqref{eq:control_sigma_i}, we have
\[
\| \bD_1 - \lambda_1(\sigma) \id_d \|_{\op} = o_{d, \P}(1). 
\]
Moreover, we have
\[
\| \bW^\sT \bW \|_{\op} =\| \bW \bW^\sT \|_{\op} \leq \| \sqrt{d} \bGamma^{1/2} \|_{\op}^2 \| \bG \bG^\sT \|_{\op} = O_{d, \P}(1),
\]
where we denoted by $\bG$ the matrix with columns $\bg_ i \sim \normal ( \bzero, \id_d/d)$. Hence, we have 
\begin{equation}
\| \bT_1 -  \lambda_1^2 \bW^\sT \bW \|_{\op} \le \| \bD_1 - \lambda_1 \id_d \|_{\op} \| \bW^\sT \bW \|_{\op} ( \| \bD_1 \|_{\op} +1 ) = o_{d, \P}(1).
\label{eq:U_bound_2_QF}
\end{equation}

\noindent
{\bf Step 5. Term $\bT_2$. } We have
\[
\bT_2 = (\zeta_2(\sigma_i) \zeta_2(\sigma_j)\< \bu_i, \bu_j\>^2/2)_{i, j \in [N]} = \bD_2 (\<\bw_i, \bw_j\>^2/2)_{i, j \in [N]} \bD_2,
\]
where 
\[
\bD_2 = \diag( (\zeta_2(\sigma_i)) / \| \bw_i \|_2^2 ).
\]
By the uniform convergence of $\zeta_2(\sigma_i)$ to $\lambda_2(\sigma)$, we have
\[
\| \bD_2 - \lambda_2 \id_d \|_{\op} = o_{d, \P}(1).
\]
Moreover, we have (see below)
\[
\| (\<\bw_i, \bw_j\>^2)_{i, j \in [N]} \|_{\op} = O_{d, \P}(1).
\]
Hence, we have 
\[
\| \bT_2 -  \lambda_2^2 (\< \bw_i, \bw_j\>^2/2)_{i, j \in [N]} \|_{\op} \le \| \bD_2 - \lambda_2 \id_d \|_{\op} \| (\< \bw_i, \bw_j\>^2/2)_{i, j \in [N]} \|_{\op} (\| \bD_2 \|_{\op} + 1) = o_{d, \P}(1).
\]
Moreover, by the estimates in proof of Theorem 2.1 in \cite{el2010spectrum}, we have 
\[
\| (\< \bw_i, \bw_j\>^2/2)_{i, j \in [N]} - [\tr(\bGamma^2)/2] \ones \ones^\sT - (1/2) \id_N \|_{\op} = o_{d, \P}(1). 
\]
Hence, we get 
\begin{equation}
\| \bT_2 -  \lambda_2^2 [\tr(\bGamma^2)/2] \ones \ones^\sT - [\lambda_2^2/2] \id_N \|_{\op} = o_{d, \P}(1).
\label{eq:U_bound_3_QF}
\end{equation}

\noindent
{\bf Step 6. Term $\sum_{k\ge 3}\ddiag(\bT_k)$. } Denote $\ddiag(\bT_k)$ the diagonal matrix composed of diagonal entries of $\bT_k$. We have 
\[
\begin{aligned}
&\Big\vert \sum_{k\ge 3} ( (T_k)_{ii} - \lambda_{k}(\sigma)^2 / k!) \Big\vert = \Big\vert \| \sigma_i \|_{L^2}^2 - \sum_{k = 0}^2 \zeta_k(\sigma_i)^2 / k! - \| \sigma \|_{L^2}^2 + \sum_{k = 0}^2 \lambda_k(\sigma)^2 / k! \Big\vert \\
\le& \| \sigma - \sigma_i \|_{L^2} [ 2\| \sigma \|_{L^2} + \| \sigma - \sigma_i \|_{L^2} ] + \sum_{k = 0}^2 \vert \zeta_k(\sigma_i)^2 - \lambda_k(\sigma)^2 \vert / k!.
\end{aligned}
\]
Note that we have shown (cf Eq.~\eqref{eq:control_sigma_i})
\[
\sup_{i \in [N]} \max\Big\{ \| \sigma - \sigma_i \|_{L^2}, \max_{k = 0, 1, 2} \vert \zeta_k(\sigma_i) - \lambda_k(\sigma) \vert \Big\} = o_{d, \P}(1). 
\]
Therefore, we have 
\begin{equation}
\Big\| \sum_{k \ge 3} \ddiag(\bT_k) - ( \Tilde \lambda - \lambda_2^2 / 2) \id_N \Big\|_{\op} = o_{d, \P}(1).
\label{eq:U_bound_4_QF}
\end{equation}

\noindent
{\bf Step 7. Term $\sum_{k\ge 3} [ \bT_k - \ddiag(\bT_k)] $. } We have 
\[
\begin{aligned}
&\Big\Vert \sum_{k \ge 3} [ \bT_k - \ddiag(\bT_k)] \Big\Vert_F \le \sum_{k \ge 3} \| \bT_k - \ddiag(\bT_k) \|_F \\
\le& \sum_{k\ge 3} \Big[ \Big( \sum_{i, j = 1}^N \zeta_k(\sigma_i)^2 \zeta_k(\sigma_j)^2 \Big) \Big(\sup_{i \neq j }\< \bu_i, \bu_j \>^{2k} / (k!)^2 \Big) \Big]^{1/2}\\
\le& \Big[ \sum_{k \ge 3} \sum_{i=1}^N \zeta_k(\sigma_i)^2 / k! \Big] \max_{i \neq j} \< \bu_i, \bu_j\>^3  \\
\le& \| \sigma_i \|_{L^2}^2  \times N  \max_{i \neq j} \< \bu_i, \bu_j\>^3.
\end{aligned}
\]
Note we have $\max_{i \in [N]} \| \sigma_i \|_{L^2}^2 = O_{d, \P}(1)$. Moreover, we have (see for example Lemma 10 in \cite{ghorbani2019linearized})
\[
 \max_{i \neq j} \< \bu_i, \bu_j\>^3= \tilde O_{d, \P}(d^{- 3/2}). 
\]
Therefore, we have 
\begin{equation}
\Big\Vert \sum_{k \ge 3} [ \bT_k - \ddiag(\bT_k) ] \Big\Vert_F = o_{d, \P}(1). 
\label{eq:U_bound_5_QF}
\end{equation}

Combining the bounds \eqref{eq:U_bound_1_QF}, \eqref{eq:U_bound_2_QF}, \eqref{eq:U_bound_3_QF}, \eqref{eq:U_bound_4_QF} and \eqref{eq:U_bound_5_QF} into the decomposition \eqref{eq:decomposition_U_QF} proves the lemma. 
\end{proof}

\subsubsection{Approximation of the $\bV$ vector}

\begin{lemma}\label{lem:QF_V}
Under the assumptions of Theorem \ref{thm:RF}, define $\bV = (V_1, \ldots, V_N)^\sT$ with
\[
V_i = \E_\bx[f_*(\bx) \sigma(\< \bw_i, \bx\>)]
\]
where $(\bw_i)_{i \in [N]} \sim \normal(\bzero, \bGamma)$ independently. Then as $N/d = \rho$ with $d \to \infty$, we have 
\[
\| \bV - \tau  \ones / \sqrt{d} \|_2^2 = \| \bB \|_F^2 \cdot o_{d, \P}( 1 ),
\]
where
\[
\tau = \sqrt{d} \cdot \lambda_2\tr(\bB \bGamma).
\]
\end{lemma}

\begin{proof}[Proof of Lemma \ref{lem:QF_V}] Without loss of generality, we assume $\| \bB \|_F = 1$ in the proof (it suffices to divide $V_i$ by $\| \bB \|_F$). Consider $\bw_i \in \R^d$. Take $\bR$ to be an orthogonal matrix such that $\bR \bw_i = \| \bw_i \|_2 \be_1$, then we have   
\[
\begin{aligned}
V_i =& \E_\bx [ f_* ( \bR^\sT \bx ) \sigma (\| \bw_i \|_2 x_1 ) ] \\
= & \E_\bx [ ( \< \bx , \bR \bB \bR^\sT \bx \> - \tr (\bB) ) \sigma (\| \bw_i \|_2 x_1 ) ] \\
= & \E_{x_1} \Big[ \Big(  x_1^2  \frac{\< \bw_i , \bB \bw_i \>}{\| \bw_i \|_2^2} + \tr ( \bP_{\perp \bw_i} \bB) - \tr (\bB) \Big) \sigma (\| \bw_i \|_2 x_1 )  \Big] \\
=& \E_{x_1} \Big[   (x_1^2 - 1)  \frac{\< \bw_i , \bB \bw_i \>}{\| \bw_i \|_2^2}  \sigma (\| \bw_i \|_2 x_1 )  \Big] \\ 
\equiv & \frac{\< \bw_i , \bB \bw_i \>}{\| \bw_i \|_2^2} \zeta_2 ( \sigma_i ),
\end{aligned}
\]
where $\bP_{\perp \bw_i}$ is the projection on the hyperplane orthogonal to $\bw_i$, and we recall the definition of $\zeta_2 ( \sigma_i) $ of Lemma \ref{lem:RF_QF_kernel}:
\[
\zeta_2 ( \sigma_i ) =\E_G [(G^2 - 1) \sigma ( \| \bw_i \|_2 G ) ],
\]
with $G$ a standard normal random variable. 

We define the following interpolating variables:
\[
V^{(1)}_i =  \frac{\< \bw_i , \bB \bw_i \>}{\| \bw_i \|_2^2} \lambda_2 , \qquad V^{(2)}_i =  \< \bw_i , \bB \bw_i \> \lambda_2, \qquad V^{(3)}_i = \tr( \bGamma \bB ) \lambda_2,
\]
and the associated vectors $\bV^{(1)}$, $\bV^{(2)}$ and $\bV^{(3)}$. We bound successively the distance between these vectors. We will denote by $\bP_{\bw_i}$ the projection onto vector $\bw_i$. First, we consider:
\[
\| \bV - \bV^{(1)} \|_2^2 = \sum_{i=1}^N \tr ( \bP_{\bw_i} \bB )^2 ( \zeta_2 (\sigma_i) - \lambda_2 )^2.
\]
One can check, using a similar argument as for Eq.~\eqref{eq:conv_lambda_2_dom} and dominated convergence, that
\begin{equation}
\lim_{t \to 1 } \frac{\E [ (G^2 - 1 ) ( \sigma (t G) - \sigma (G) ) ]}{t - 1} = \lambda_4 (\sigma) + 2 \lambda_2 ( \sigma).
\label{eq:diff_lambda_2_RF_QF}
\end{equation}
Hence, recalling \eqref{eq:sup_wi_RF_QF}, we have 
\begin{equation}
\| \bV - \bV^{(1)} \|_2^2 = O_{d,\P} \Big( \Big( \sup_{i\in[N]} \tr ( \bP_{\bw_i} \bB ) \Big)^2  \sum_{i = 1}^N ( \|\bw_i \|_2 - 1 )^2 \Big).
\label{eq:V_V1_diff_exp_RF_QF}
\end{equation}
Let us first show that the sum is bounded with high probability: denoting $\bg \sim \normal ( \bzero , \id_d ) $, classical sub-Gaussian concentration inequalities (see for example Theorem 6.3.2 in \cite{vershynin2010introduction}) shows that
\begin{equation}
\Big\| \| \bGamma^{1/2} \bg \|_2 - \| \bGamma^{1/2} \|_F \Big\|_{\psi_2} \leq C \| \bGamma^{1/2} \|_{\text{op}},
\label{eq:subGaussian_orlicz_RF_QF}
\end{equation}
where $\| \cdot \|_{\psi_2}$ denotes the sub-Gaussian Orlicz norm. By assumption, we have $\| \bGamma^{1/2} \|_{\text{op}} = \| \bGamma \|_{\text{op}}^{1/2} = O_d (d^{-1/2})$, and $\| \bGamma^{1/2} \|_F = \sqrt{\tr \bGamma} = 1$. Hence, for $\bw_i \sim \normal ( \bzero , \bGamma ) $, we have 
\begin{equation}
\Big\| \sqrt{d} \| \bw_i \|_2 - \sqrt{d} \Big\|_{\psi_2} \leq C.
\label{eq:orlicz_normwi_RF_QF}
\end{equation}
Therefore, we have
\begin{equation}
\sum_{i = 1}^N ( \|\bw_i \|_2 - 1 )^2 = O_{d, \P}(1). 
 \label{eq:bound_sum_RF_QF}
\end{equation}
Furthermore, we readily have (for example from \eqref{eq:sup_wi_RF_QF})
\begin{equation}
    \sup_{i\in [N]} \| \bw_i \|^{-4} = O_{d,\P} (1).
    \label{eq:bound_inv_wi_RF_QF}
\end{equation}
Noticing that $\tr ( \bw_i \bw_i^\sT \bB) = \| \bB^{1/2} \bw_i \|^2_2$ and by the same argument as for \eqref{eq:subGaussian_orlicz_RF_QF}, we have:
\begin{equation}
\Big\| \| \bB^{1/2} \bGamma^{1/2} \bg \|_2 - \E[\| \bB^{1/2}  \bGamma^{1/2} \bg \|_2 ]  \Big\|_{\psi_2} \leq C \| \bB^{1/2} \bGamma^{1/2} \|_{\text{op}}.
\label{eq:orlicz_Bwi_RF_QF}
\end{equation}
By assumption {\bf A2}, we have $\| \bB^{1/2} \bGamma^{1/2} \|_{\text{op}} \leq \| \bB^{1/2}  \|_{\text{op}} \| \bGamma^{1/2} \|_{\text{op}}= O_d ( d^{-1/2})$ and 
\[
\E [ \| \bB^{1/2}  \bGamma^{1/2} \bg \|_2 ] \leq ( \E [ \| \bB^{1/2}  \bGamma^{1/2} \bg \|_2^2])^{1/2} = \tr ( \bGamma \bB )^{1/2} \leq   \|\bGamma \|_F^{1/2} \| \bB \|_F^{1/2} \leq \| \bGamma \|_{\op}^{1/4} \tr ( \bGamma)^{1/4} = O_d (d^{-1/2}),
\]
which combined with \eqref{eq:orlicz_Bwi_RF_QF} yields
\begin{equation}
\sup_{i \in [N]} \| \bB^{1/2} \bw_i \|_2^2 = o_{d,\P} (1).
\label{eq:sup_Bwi_RF_QF}
\end{equation}
Combining the bounds \eqref{eq:bound_sum_RF_QF}, \eqref{eq:bound_inv_wi_RF_QF} and \eqref{eq:sup_Bwi_RF_QF} into \eqref{eq:V_V1_diff_exp_RF_QF}, we get
\begin{equation}
\| \bV - \bV^{(1)} \|^2_2 = o_{d,\P} (1).
\label{eq:V_V1_RF_QF}
\end{equation}
Consider now 
\begin{equation}\label{eqn:bound_V1_v2}
\begin{aligned}
\| \bV^{(1)} - \bV^{(2)} \|_2^2 = &\sum_{i=1}^N \lambda_2^2 \< \bw_i , \bB \bw_i \>^2 \Big( \frac{1}{\| \bw_i \|_2^2} - 1 \Big)^2 \\
\leq &\lambda_2^2 \Big( \sup_{i \in [N]} \| \bB^{1/2} \bw_i \|^2_2 / \| \bw_i \|_2^2 \Big) \sum_{i=1}^N  (\|\bw_i \|_2^2 - 1)^2.
\end{aligned}
\end{equation}
We have
\[
\begin{aligned}
\E_{\bw_i \sim \normal (\bzero , \bGamma ) } [(\|\bw_i \|_2^2 - 1)^2] = & \E_{\bg \sim \normal (\bzero , \id ) } [(\< \bg \bg^\sT , \bGamma \> - \tr( \bGamma ) )^2] = 2 \| \bGamma \|^2_F =  O_{d , \P} (d^{-1}).
\end{aligned}
\]
Hence we must have
\[
\sum_{i=1}^N  (\|\bw_i \|_2^2 - 1)^2 = O_{d, \P}(1),
\]
which, combined with \eqref{eq:sup_Bwi_RF_QF} and \eqref{eqn:bound_V1_v2}, yields
\begin{equation}
\| \bV^{(1)} - \bV^{(2)} \|_2^2 =  o_{d,\P} (1).
\label{eq:V1_V2_RF_QF}
\end{equation}
Consider the last comparison:
\[
\begin{aligned}
\| \bV^{(2)} - \bV^{(3)} \|_2^2 = &\sum_{i=1}^N \lambda_2^2 \Big( \< \bw_i , \bB \bw_i \> - \tr ( \bGamma \bB ) \Big)^2.
\end{aligned}
\]
Taking the expectation:
\[
\begin{aligned}
\E_{\bw_i \sim \normal (\bzero , \bGamma ) } [( \< \bw_i , \bB \bw_i \> - \tr ( \bGamma \bB ) )^2] = & \E_{\bg \sim \normal (\bzero , \id ) } [(\< \bg \bg^\sT , \bGamma^{1/2} \bB \bGamma^{1/2} \> - \tr( \bGamma \bB) )^2] \\
=& 2 \| \bGamma^{1/2} \bB \bGamma^{1/2} \|^2_F \\
\leq &  2 \| \bGamma \|^2_{\text{op}} \| \bB \|^2_F=  O_{d} (d^{-2}).
\end{aligned}
\]
We conclude that
\[
\sum_{i=1}^N \Big( \< \bw_i , \bB \bw_i \> - \tr ( \bGamma \bB ) \Big)^2 = o_{d,\P} (1),
\]
and therefore
\begin{equation}
\| \bV^{(2)} - \bV^{(3)} \|_2^2 =  o_{d,\P} (1),
\label{eq:V2_V3_RF_QF}
\end{equation}
where $\bV^{(3)} = \lambda_2 \tr (\bGamma \bB ) \ones  $. Combining the above three bounds \eqref{eq:V_V1_diff_exp_RF_QF}, \eqref{eq:V1_V2_RF_QF} and \eqref{eq:V2_V3_RF_QF} yields the desired result.
\end{proof}

\subsubsection{Calculating $\ones^\sT \bU_0^{-1} \ones / d$}

The following proposition is stated in slightly more general terms, in order to be used in both the proofs of Theorem \ref{thm:RF} and Theorem \ref{thm:RFMixture}.

\begin{proposition}\label{prop:one_U_one_expression}
Let $(\bw_i)_{i \in [N]} \sim \normal(\bzero, \bGamma)$ independently, where $\bGamma$ satisfies assumption {\bf A2} (resp. {\bf B2}). Denote by $\lambda_k =\E_{G \sim \normal(0, 1)}[\sigma(G) \He_k(G)]$ the $k$-th Hermite coefficient of $\sigma$. Define $\tilde{\lambda} = \E_{G \sim \normal(0, 1)}[\sigma(G)^2] - \lambda^2_1$. Consider $\kappa \equiv \kappa (d)$ positive constants that are uniformly upper bounded. Define 
\[
\bU_0 = \bA_0 +  \kappa \ones \ones^\sT / d + \bmu \bmu^\sT,
\]
where 
\[
\begin{aligned}
\bA_0 =& \tilde \lambda \id_N + \lambda_1^2 \bW^\sT \bW, \\
\mu_i =& \lambda_2 (\| \bw_i \|_2^2 - 1)/ 2.
\end{aligned}
\]

Then we have 
\[
\< \ones, \bU_0^{-1} \ones\> /d = \psi / (1 + \kappa \psi) + o_{d, \P}(1),
\]
where $\psi>0$ is the unique solution of 
\begin{align}
- \tilde{\lambda}= - \frac{\rho}{\psi} + \int \frac{\lambda_1^2 t}{1 + \lambda_1^2 t \psi} \mathcal{D}(\de t)\,,
\end{align}
where $\mathcal{D}$ is the empirical distribution of eigenvalues of $d \cdot \bGamma$. 
\end{proposition}

The proof of Proposition \ref{prop:one_U_one_expression} is a direct combination of Lemma \ref{lem:inner_prod_limit}, \ref{lem:expression_for_one_U_one}, and \ref{lem:rand_quad_form} below. 

\begin{lemma} \label{lem:inner_prod_limit}
Let $(\bw_i)_{i \in [N]} \sim \normal(\bzero, \bGamma)$ independently. Assume condition {\bf A2} holds (resp. {\bf B2}). Let $\bmu = (\| \bw_i \|_2^2 - 1)_{i \in [N]}$, and $\bA_0 = c_1 \id_N + c_2 \bW^\sT \bW$, where $c_1 \equiv c_1 (d)$ and $c_2 \equiv c_2 (d)$ are constants that are asymptotically upper and lower bounded by strictly positive constants. Then as $d \rightarrow \infty$ and $N/d \to \rho$, we have 
\begin{align}
\<\ones, \bA_0^{-1} \bmu \> / \sqrt{d} = o_{d, \P}(1). 
\end{align}
\end{lemma}

\begin{proof} We first prove the lemma under the following extra assumption on the covariance matrix: there exists a (fixed) integer $K$ such that 
\begin{equation}
\bGamma = \bQ \diag(\gamma_1 \id_{d_1}, \ldots, \gamma_K \id_{d_K}) \bQ^\sT,
\label{eq:block_gamma}
\end{equation}
for some orthogonal matrix $\bQ$ and $d \cdot \gamma_i \leq C$. Furthermore, there exists an $\eps > 0$ such that $d_k/d \ge \eps$ for $d$ sufficiently large. 

Without loss of generality, we assume $\bGamma = \diag(\gamma_1 \id_{d_1}, \ldots, \gamma_K \id_{d_K})$, and we divide $\bw_i$ into vectors corresponding to each block
\[
\bw_i = (\bw_{i, 1}; \ldots; \bw_{i, K}) \in \R^d,
\]
where $\bw_{i, k} \in \R^{d_k}$, and we denote $\bW_k = [\bw_{1, k}, \bw_{2, k}, \ldots, \bw_{N, k}] \in \R^{d_k \times N}$ for $k \in [K]$. 

\noindent
{\bf Step 1. Decouple the randomness. }

Let $(\tilde \bw_i)_{i \in [N]} \sim \normal(\bzero, \bGamma)$ independently and independent of $(\bw_i)_{i \in [N]}$. We divide $\tilde \bw_i$ into segments corresponding to each blocks
\[
\tilde \bw_i = (\tilde \bw_{i, 1}; \ldots; \tilde \bw_{i, K}),
\]
where $\tilde \bw_{i, k} \in \R^{d_k}$, and we denote $\tilde \bW_k = [\tilde \bw_{1, k}, \tilde \bw_{2, k}, \ldots, \tilde \bw_{N, k}] \in \R^{d_k \times N}$ for $k \in [K]$. 

Define 
\[
\begin{aligned}
\bD_{k, \bw} =& \diag(\| \bw_{1, k}\|_2, \ldots, \| \bw_{N, k} \|_2) \in \R^{N\times N}, \\
\bD_{k, \tilde \bw} =& \diag(\| \tilde \bw_{1, k}\|_2, \ldots, \| \tilde \bw_{N, k} \|_2)\in \R^{N\times N}.
\end{aligned}
\]
Using the fact that $\| \bg \|_2$ is independent of $\bg / \| \bg \|_2$ for $\bg \sim \normal(\bzero, \id)$, the following two sets of random variables have the same distribution: 
\[
\Big\{ (\bW_k^\sT \bW_k)_{k \in [K]}, (\| \bw_{ik} \|_{2})_{i \in [N],k \in [K]} \Big\} \stackrel{{\rm d}}{=} \Big\{  (\bD_{k, \bw} \bD_{k, \tilde \bw}^{-1}\tilde \bW_k^\sT \tilde \bW_k \bD_{k ,\tilde \bw}^{-1}\bD_{k, \bw})_{k \in [K]}, (\| \bw_{ik} \|_{2})_{i \in [N],k \in [K]} \Big\}.
\]

Define 
\[
\bar \bA_0 = c_1 \id_d + c_2 \sum_{k \in [K]} \bD_{k, \bw} \bD_{k, \tilde \bw}^{-1}\tilde \bW_k^\sT \tilde \bW_k \bD_{k ,\tilde \bw}^{-1}\bD_{k, \bw}. 
\]
Then we have 
\begin{equation}
\<\ones, \bA_0^{-1} \bmu \>/\sqrt d \stackrel{{\rm d}}{=} \<\ones, \bar \bA_0^{-1} \bmu \>/\sqrt d. 
\label{eq:lemma_inverse_bound_1}
\end{equation}

\noindent
{\bf Step 2. Bound the difference between $\bar \bA_0$ and $\tilde \bA_0$. }

Define
\[
\tilde \bA_0 = c_1 \id_d + c_2 \sum_{k \in [K]} \tilde \bW_k^\sT \tilde \bW_k. 
\]
Since $d_k \to \infty$ as $d \to \infty$, we have 
\[
\| \bD_{k ,\tilde \bw}^{-1}\bD_{k, \bw} - \id_{N} \|_{\op} = o_{d, \P}(1), 
\]
and hence 
\[
\| \tilde \bA_0 - \bar \bA_0 \|_{\op} \le 2 c_2 \sum_{k \in [K]} \| \bD_{k, \bw} \bD_{k, \tilde \bw} - \id_d \|_{\op} \| \tilde \bW_k^\sT \tilde \bW_k \|_{\op} \|\bD_{k, \bw} \bD_{k, \tilde \bw} \|_{\op} = o_{d, \P}(1). 
\]
By definition, $\tilde \bA_0, \bar \bA_0 \succeq c_1 \id$ and therefore $\| \tilde \bA_0^{-1} \|_{\op} , \| \bar \bA_0^{-1} \|_{\op} = O_{d,\P} (1)$. We deduce
\[
\| \tilde \bA_0^{-1} - \bar \bA_0^{-1} \|_{\op} =\| \tilde \bA_0^{-1} (\bar \bA_0 - \tilde \bA_0)  \bar \bA_0^{-1} \|_{\op}  = o_{d, \P}(1).
\]
This gives (recalling that $\| \bmu \|_2^2 = O_{d,\P} (1)$)
\begin{equation}
\< \ones, \bar \bA_0^{-1} \bmu\> / \sqrt d - \< \ones, \tilde \bA_0^{-1} \bmu \> / \sqrt d = o_{d, \P}(1). 
\label{eq:lemma_inverse_bound_2}
\end{equation}

\noindent
{\bf Step 3. Calculating the second moment of $\< \ones, \tilde \bA_0^{-1} \bmu\> / \sqrt d$. }

Since we have 
\[
\E_{\bW}[(\< \ones, \tilde \bA_0^{-1} \bmu\> / \sqrt d)^2] = \< \ones, \tilde \bA_0^{-2} \ones\> /d \cdot \E_{\bw \sim \normal(\bzero, \bGamma)}[(\| \bw \|_2^2 - 1)^2]. 
\]
Note that
\[
\E_{\bw \sim \normal(\bzero, \bGamma)}[(\| \bw \|_2^2 - 1)^2] = O_{d, \P}(1/d),
\]
and using that $\| \tilde \bA_0^{-1} \|_{\op}  = O_{d,\P} (1)$,
\[
\< \ones, \tilde \bA_0^{-2} \ones\> /d = O_{d, \P}(1).
\]
Therefore
\[
\E_{\bW}[(\< \ones, \tilde \bA_0^{-1} \bmu\> / \sqrt d)^2] = o_{d, \P}(1).
\]
By Chebyshev inequality we have
\begin{equation}
\< \ones, \tilde \bA_0^{-1} \bmu\> / \sqrt d = o_{d, \P}(1). 
\label{eq:lemma_inverse_bound_3}
\end{equation}
Combining \eqref{eq:lemma_inverse_bound_1}, \eqref{eq:lemma_inverse_bound_2} and \eqref{eq:lemma_inverse_bound_3} proves the lemma in the case of a covariance of the form \eqref{eq:block_gamma}:
\begin{align}
\<\ones, \bA_0^{-1} \bmu \> / \sqrt{d} = o_{d, \P}(1). 
\label{eq:result_block}
\end{align}

\noindent
{\bf Step 4. From discrete to continuous spectrum. }

We consider $\bGamma$ a covariance matrix verifying assumption {\bf A2}. For a given $\eps > 0 $ and $K$ sufficiently large, we consider $\bGamma_{\eps}$ a matrix obtained from $\bGamma$ by binning its eigenvalues to at most $K$ points of $[0,C/d]$, such that we have $\tr ( \bGamma_{\eps}) = 1$ and $\lim_{d \to \infty} d \cdot \| \bGamma - \bGamma_{\eps} \|_{\op} \leq \eps$ (recall that $\| \bGamma \|_{\op} \leq C/d$ by assumption). Such a matrix always exists from the condition $\tr(\bGamma) = 1$ and the weak convergence of the spectrum of $d \cdot \bGamma$. 

By construction $\bGamma_\eps$ is of the form \eqref{eq:block_gamma}.
Consider $\bG = ( \bg_1 , \ldots , \bg_N ) \in \R^{d \times N}$ where $\bg_i \sim_{i.i.d.} \normal ( \bzero , \id_d)$. We define:
\[
\begin{aligned}
\bmu &= ( \| \bGamma^{1/2} \bg_i \|_2^2  -1 )_{i \in [N]}, \qquad & \bmu_{\eps} = ( \| \bGamma_{\eps}^{1/2} \bg_i \|_2^2  -1 )_{i \in [N]}, \\
\bA_{0} &= c_1 \id_d + c_2 \bG^\sT \bGamma \bG, \qquad & \bA_{0,\eps} = c_1 \id_d + c_2 \bG^\sT \bGamma_{\eps} \bG.
\end{aligned}
\]
We have  for $d$ sufficiently large,
\[
\| \bA_{0} - \bA_{0,\eps} \|_{\op} = \| \bG^\sT ( \bGamma - \bGamma_{\eps}) \bG \|_{\op} \leq \| \bG \|^2_{\op} \| \bGamma - \bGamma_{\eps} \|_{\op} \leq 2 \eps \| \bG \|^2_{\op}/d.
\]
Furthermore, using $\tr ( \bGamma - \bGamma_{\eps} ) = 0$, we have
\[
\E [ \| \bmu - \bmu_{\eps} \|_2^2 ] = N \E [ ( \< \bg_i \bg_i^\sT , \bGamma - \bGamma_{\eps} \> )^2 ] = 2N \| \bGamma - \bGamma_{\eps} \|_F^2 \leq 2  \rho \eps^2.
\]
Therefore
\[
\begin{aligned}
\Big\vert \<\ones , \bA^{-1}_0 \bmu - \bA^{-1}_{0,\eps} \bmu_{\eps} \> / \sqrt{d} \Big\vert \leq & \Big\vert  \<\ones , \bA^{-1}_0 ( \bA_{0,\eps} - \bA_0 ) \bA_{0,\eps}^{-1} \bmu  \> / \sqrt{d} \Big\vert + \Big\vert  \<\ones , \bA_{0.\eps}^{-1} ( \bmu_{\eps}  - \bmu)  \> / \sqrt{d} \Big\vert \\
\leq& \| \bA_0^{-1} \|_{\op} \| \bA_0 - \bA_{0,\eps} \|_{\op} \| \bA_{0,\eps}^{-1} \|_{\op} \| \bmu \|_2 + \| \bA_{0,\eps}^{-1} \|_{\op}  \| \bmu - \bmu_{\eps} \|_{2}.
\end{aligned}
\]
Noticing that $\| \bA_0^{-1} \|_{\op} , \| \bA_{0,\eps}^{-1} \|_{\op} \leq c_1^{-1}$, and using \eqref{eq:result_block} applied to $\bGamma_{\eps}$, we get for $d$ sufficiently large:
\begin{equation}
\Big\vert \<\ones , \bA^{-1}_0 \bmu \> / \sqrt{d} \Big\vert \leq o_{d,\P} (1) + 2 \eps c_1^{-2} \| \bmu \|_2 \| \bG \|^2_{\op}/d  + c_1^{-1}  \| \bmu - \bmu_{\eps} \|_{2}.
\label{eq:bA_decomposition_continuous}
\end{equation}
We have $\| \bmu \|_2 \| \bG \|^2_{\op}/d = O_{d,\P} (1)$ hence for any $\delta > 0$ there exists a constant $C_{\delta}$ (which do not depend on $\eps$) such that:
\[
\P ( \eps \| \bmu \|_2 \| \bG \|^2_{\op}/d > \eps C_{\delta}) \leq \delta.
\]
Taking a sequence $\delta \to 0$ and  $\eps$ such that $\eps \propto C^{-1}_{\delta}$ shows that this is equivalent to 
\begin{equation}
\eps \| \bmu \|_2 \| \bG \|^2_{\op}/d = o_{d, \P} (1).
\label{eq:bA_bound1}
\end{equation}
 By Markov inequality, 
\[
\lim_{d \to \infty} \P ( \| \bmu - \bmu_{\eps} \|_{2} \geq \eps \sqrt{2 \rho/\delta}  ) \leq \delta.
\]
Taking $\eps \propto \sqrt{\delta}$, we deduce that this is equivalent to 
\begin{equation}
\| \bmu - \bmu_{\eps} \|_{2} = o_{d,\P} (1) .
\label{eq:bA_bound2}
\end{equation}
Substituting \eqref{eq:bA_bound1} and \eqref{eq:bA_bound2} in \eqref{eq:bA_decomposition_continuous} concludes the proof.
\end{proof}

\begin{lemma}\label{lem:expression_for_one_U_one}
Under the same setting as Proposition \ref{prop:one_U_one_expression}, we have
\begin{align}\label{eqn:one_U_one_expression}
    \<\ones, \bU_0^{-1} \ones\>/d = \frac{\ones^\sT \bA_0^{-1} \ones / d}{1 + \kappa \ones^\sT \bA_0^{-1} \ones / d} + o_{d, \P}(1). 
\end{align}
\end{lemma}

\begin{proof}[Proof of Lemma \ref{lem:expression_for_one_U_one}. ]

Define $\bz = \sqrt{\kappa} \ones / \sqrt{d}$. Then we have 
\[
\bU_0 = \bA_0  + \bz\bz^\sT + \bmu \bmu^\sT. 
\]
By assumption, we have $ \kappa = O_{d, \P}(1)$ and therefore $ \| \bz \|_2 = O_{d, \P}(1)$. We have already seen that $ \| \bA_0^{-1} \|_{\op} ,  \| \bA_0^{-1} \|_{\op} = O_{d, \P}(1)$. Furthermore
\[
\| \bA_0 \|_{\op} \leq \tilde \lambda  + \lambda_1^2 \lambda_{\max} ( \bW^\sT \bW ) = O_{d,\P} (1) . 
\]

By Sherman Morrison Woodbury formula, we have
\[
\begin{aligned}
\ones^\sT \bU_0^{-1} \ones /d =&  \ones^\sT \bA_0^{-1} \ones /d- \ones^\sT \bA_0^{-1} [\bz, \bmu] (\id_2 + [\bz, \bmu]^\sT \bA_0^{-1} [\bz, \bmu])^{-1} [\bz, \bmu]^\sT \bA_0^{-1} \ones/d.
\end{aligned}
\]
Note that by
\[
\| (\id_2 + [\bz, \bmu]^\sT \bA_0^{-1} [\bz, \bmu])^{-1} \|_{F} = O_{d, \P}(1), 
\]
and by Lemma \ref{lem:inner_prod_limit}, we have (since $\bz^\sT \bA_0^{-1} \bmu, \ones^\sT \bA_0^{-1} \bmu / \sqrt d = o_{d, \P}(1)$)
\[
\begin{aligned}
&\ones^\sT \bA_0^{-1} [\bz, \bmu] (\id_2 + [\bz, \bmu]^\sT \bA_0^{-1} [\bz, \bmu])^{-1} [\bz, \bmu]^\sT \bA_0^{-1} \ones/d\\
=& (\ones^\sT \bA_0^{-1} \bz)^2 (1 + \bz^\sT \bA_0^{-1} \bz)^{-1} / d + o_{d, \P}(1) = \kappa (\ones^\sT \bA_0^{-1} \ones / d)^2 (1 + \kappa \ones^\sT \bA_0^{-1} \ones/d)^{-1} + o_{d, \P}(1) .
\end{aligned}
\]
This proves the lemma. 
\end{proof}

In the following, we give an asymptotic expression for $\< \ones, \bA_0^{-1} \ones\> / d$.

\begin{lemma} \label{lem:rand_quad_form}
Let $(\bw_i)_{i \in [N]} \sim \normal(\bzero, \bGamma)$ independently, while $\bGamma$ satisfies assumption {\bf A2} (resp. {\bf B2}). Denote $\bW = (\bw_1, \ldots, \bw_N) \in \R^{d \times N}$. Let $\tilde \lambda$ and $\lambda_1$ be two positive constants. Define
\[
\bA_0 = \tilde \lambda \id_N + \lambda_1^2 \bW^\sT \bW. 
\]
Let $\rho \in (0, \infty)$. We have almost surely
\begin{align} \label{eq:claim_1}
\lim_{N/d = \rho, d \to \infty} \vert \ones^\sT \bA_0^{-1} \ones /d - \tr(\bA_0^{-1}) / d \vert = 0 .
\end{align}
In addition, assume $\mathcal{D}$ is the limiting spectral distribution of $d \cdot \bGamma$. Then, we have almost surely
\begin{align} 
    \lim_{N/d = \rho, d \to \infty} \frac{1}{d} \tr(\bA_0^{-1}) =  m_{\mathcal{D}} (- \tilde \lambda),
\end{align}
where $m_{\mathcal{D}}(\cdot): \mathbb{C}^+ \rightarrow \mathbb{C}^+$ is the companion Stieltjes transform associated with $\mathcal{D}$. For any $x \in \mathbb{C}^+$, $m_{\mathcal{D}}(x)$ satisfies the so called Silverstein's equation:
\begin{align} \label{eqn:silverstein}
    x = -\frac{\rho}{m_{\mathcal{D}}(x)} + \int \frac{\lambda_1^2 t}{1 + \lambda_1^2 t  m_{\mathcal{D}}(x)}\mathcal{D}({\rm d} t). 
\end{align}
\end{lemma}
\begin{proof}[Proof of Lemma \ref{lem:rand_quad_form}. ]
Consider the event $$A_N(t) := \{\vert \ones^\sT \bA_0^{-1} \ones /d - \tr(\bA_0^{-1}) / d \vert > t\}.$$
Let $\bQ \in \R^{N \times N}$ be an orthogonal matrix. By rotation invariance of Gaussian random variables, $\bQ \bW^\sT$ has the same distribution as $\bW$. In fact, by Fubini's theorem, we can draw $\bQ$ uniformly (independent of $\bA_0$) from orthogonal matrices and the distribution would still be unchanged. Let $$\tilde{A}_N(t) := \{\vert \ones^\sT (\bQ \bA_0^{-1} \bQ^\sT)^{-1} \ones /d - \tr(\bQ \bA_0^{-1}\bQ^\sT) / d \vert > t\}.$$
By the argument above, 
$$ \P[A_N(t)] = \P[\tilde{A}_N(t)].$$
Since $\bQ$ is orthogonal, $\tilde{A}_N(t)$ can be written as
\begin{align}
\{\vert \ones^\sT \bQ \bA_0^{-1}  \bQ^\sT \ones /d - \tr(\bA_0^{-1}) / d \vert > t\}.
\end{align}
Since $\bQ$ is a uniformly chosen orthogonal matrix, $\bQ^\sT \ones / \sqrt{d}$ is uniformly distributed on $\S^{N-1}(\sqrt{\rho})$, independently of $\bA_0$. Hence $\bQ^\sT \ones / \sqrt{d}$ has the same distribution as $\sqrt{\rho} \bz /\Vert \bz \Vert_2$ where $\bz \sim \normal(0, \id_{N})$. In particular, 
\begin{align}
\P[\tilde{A}_N(t)] &= \P\Big\{ \Big\vert \frac{1}{\Vert \bz \Vert^2_2} \bz^\sT \bA_0^{-1} \bz - \tr(\bA_0^{-1}) / N \Big\vert > \frac{t}{\rho} \Big\} \\
&\leq \P\Big\{ \Big\vert \frac{N}{\Vert \bz \Vert^2_2} - 1\Big\vert \bz^\sT \bA_0^{-1} \bz / N + \vert \bz^\sT \bA_0^{-1} \bz / N - \tr(\bA_0^{-1}) / N \vert > \frac{t}{\rho} \Big\} \\
&\leq P_1 + P_2,
\end{align}
where 
\begin{align*}
    P_1 = \P\Big\{\Big\vert \frac{N}{\Vert \bz \Vert^2} - 1\Big\vert \bz^\sT \bA_0^{-1} \bz / N > \frac{t}{2 \rho} \Big\}, \qquad 
    P_2 = \P\Big\{\vert \bz^\sT \bA_0^{-1} \bz / N - \tr(\bA_0^{-1}) / N \vert > \frac{t}{2\rho} \Big\}.
\end{align*}
Let's consider $P_1$ first. Since $\bA_0^{-1} \preceq  \id / \tilde \lambda $, we have
\[
\frac{\bz^\sT \bA_0^{-1} \bz }{ N} \leq \frac{1}{\tilde \lambda} \frac{\Vert \bz \Vert^2}{N},
\]
which yields
\begin{align}
P_1 &\leq \P\Big\{\Big\vert \frac{N}{\Vert \bz \Vert^2} - 1\Big\vert \frac{\Vert \bz \Vert^2}{N} > \frac{\tilde \lambda t}{2 \rho} \Big\} = \P\Big\{\Big\vert \frac{\Vert \bz \Vert^2}{N} - 1\Big\vert > \frac{\tilde \lambda t}{2 \rho} \Big\}.
\end{align}
We know due to fast concentration of $\Vert z \Vert^2 / N$ around one (see e.g. \cite{boucheron2013concentration}), $P_1$ vanish exponentially fast in $N$ (equivalently in $d$ since $N / d$ is fixed to be $\rho$).  

Now, let's consider $P_2$.  $\vert \bz^\sT \bA_0^{-1} \bz / N - \tr(\bA_0^{-1}) / N \vert$. By Hanson-Wright inequality (see e.g. \cite{boucheron2013concentration}), we have
\begin{align}
\P\bigg(  \vert \bz^\sT \bA_0^{-1} \bz / N - \tr(\bA_0^{-1}) / N \vert > \frac{t}{2 \rho} \Big\vert \bA_0 \bigg) & \leq 2 \exp \bigg\{ -c \min \Big(\frac{t^2}{\Vert \bA_0^{-1} / N \Vert_F^2}, \frac{t}{\Vert \bA_0^{-1} / N \Vert_{\op}}\Big) \bigg\} \\
&\leq 2 \exp \bigg\{ -c' \min \Big(N\tilde \lambda^2 t^2, \tilde \lambda t N \Big) \bigg\}. \label{eq:Hanson-Wright-Bound}
\end{align}
Since the bound in \eqref{eq:Hanson-Wright-Bound} is independent of $\bA_0$, it holds unconditionally. Therefore, we conclude $P_2$ vanishes exponentially fast in $N$ and $d$. We conclude that $\Pr[\tilde{A}_N(t)]$ vanishes exponentially fast as $d, N \rightarrow \infty$. Therefore, by Borel-Cantelli lemma we recover \eqref{eq:claim_1}.

Convergence of $\Tr(\bA_0^{-1}) / d$ to $m_D(- \tilde \lambda)$ is a standard result in random matrix theory. We refer the reader to \cite{bai2010spectral} Chapters 3 and 6. 
\end{proof}
\subsubsection{Proof of Theorem \ref{thm:RF}}

By Lemma \ref{lem:QF_risk_representation}, the risk has a representation 
\[
R_{\RF,N}(f_*) = \E_{\bx \sim \normal(\bzero, \id_d)}[f_*(\bx) ^2] - \bV^\sT \bU^{-1} \bV. 
\]
By Lemma \ref{lem:RF_QF_kernel}, we have 
\[
\| \bU - \bU_0 \|_{\op} = o_{d, \P}(1). 
\]
By Lemma \ref{lem:QF_V}, we have
\[
\| \bV - \tau \ones / \sqrt{d} \|_2 = \| \bB \|_F \cdot o_{d, \P}(1),
\]
where
\[
\tau = \sqrt{d} \cdot \lambda_2 \tr(\bB \bGamma). 
\]
Hence, we have 
\[
\vert \bV^\sT \bU^{-1} \bV - \tau^2 \ones^\sT \bU_{0}^{-1} \ones / d \vert =  \| \bB \|_F^2  \cdot o_{d, \P}(1). 
\]
Proposition \ref{prop:one_U_one_expression} gives the expression for 
\[
\ones^\sT \bU_0^{-1} \ones / d = \psi / (1 + \kappa \psi) +   o_{d, \P}(1), 
\]
where 
\[
\kappa = d  \cdot \lambda_2^2 \tr(\bGamma^2)/2.
\]
Hence we have 
\[
\bV^\sT \bU \bV = \tau^2 \psi / (1 + \kappa \psi) + \| \bB \|_F^2 \cdot o_{d, \P}(1). 
\]
Recalling the assumption $\E ( f_* ) = 0$, we have $\|f_* \|^2_{L^2}  = 2  \| \bB \|_F^2 $, which concludes the proof.

\subsection{Neural Tangent model: proof of Theorem \ref{thm:QF_NTK}}

Recall the definition
\[
R_{\NT, N}(f_*) = \arg\min_{\hf\in\cF_{\NT,N}(\bW)}\E\big\{(f_*(\bx)-\hf(\bx))^2\big\},
\]
where
\[
\cF_{\NT, N}(\bW) = \Big\{ f_N( \bx) = c + \sum_{i=1}^N  \sigma'(\< \bw_i, \bx\>)\< \ba_i, \bx\>: c \in \R, \ba_i \in \R^d, i \in [N] \Big\}. 
\]

\begin{proof}[Proof of Theorem \ref{thm:QF_NTK}]
We can rewrite the neural tangent model with a squared non-linearity $\sigma (x) = x^2$ as
\[
\hat f ( \bW , \bA , c ) = 2 \sum_{i=1}^N \< \bw_i, \bx \> \< \ba_i , \bx \> + c = 2 \< \bW \bA^\sT , \bx \bx^\sT \> + c. 
\]
with $\bW = [\bw_1 , \ldots , \bw_N] \in \R^{d \times N}$ and $\bA = [ \ba_1 , \ldots , \ba_N ] \in \R^{d \times N}$. Note that we have
\[
\begin{aligned}
& \E_{\bx } [ \< \bB - 2 \bW \bA^\sT , \bx \bx^\sT \> + b_0 - c )^2 ] \\
=& 2 \|  \bB - \bW \bA^\sT - \bA \bW^\sT \|^2_F  + \tr ( \bB - 2 \bW\bA^\sT )^2 - 2 \tr( \bB - 2 \bW\bA^\sT ) (c - b_0) + (c - b_0)^2,
\end{aligned}
\]
which, after minimizing over $c \in \R$, simplifies to:
\[
\min_{c \in \R} \| f_* - \hat f ( \bW , \bA , c ) \|^2_{L^2} = 2 \|  \bB - \bW \bA^\sT - \bA \bW^\sT \|^2_F. 
\]
For $\bw_i \sim \normal ( \bzero , \id_d ) $, we have $\text{rank} (\bW ) =\min (d , N) \equiv r$ with probability one. Let $\bW = \bP_1 \bS \bV^\sT$ be the singular value decomposition of $\bW$, with $\bP_1 \in \R^{d \times r}$, $\bS \in \R^{r \times r}$ and $\bV \in \R^{N \times r}$. Defining $\bG =  \bS \bV^\sT \bA \in \R^{r \times d}$, we get almost surely
\[
\min_{\bA \in \R^{d \times N}, c\in \R} \| f_* - \hat f ( \bW , \bA , c ) \|^2_{L^2} = \min_{\bG \in \R^{r \times d}} 2 \| \bB - \bP_1 \bG - \bG^\sT \bP_1^\sT \|_F^2.
\]
In the case $N \geq d$, we can take $\bG = \bP_1^\sT \bB / 2$ and we get almost surely over $\bW \in \R^{d \times N}$
\[
R_{\NT,N} (f_*) = 0.
\]
Consider the case when $N < d$, we define $\bP_2 \in \R^{d \times (d - N)}$ the completion of $\bP_1 $ to a full basis $\bP = [ \bP_1 , \bP_2] \in \R^{d \times d}$. We define $\bG_1 = \bG \bP_1 \in \R^{N \times N}$ and $\bG_2 = \bG \bP_2 \in \R^{N \times (d-N)}$ and we perform our computation in the $\bP$ basis. We have
\[
\bB - \bP_1 \bG - \bG^\sT \bP_1^\sT = \begin{pmatrix} 
\bB_{11} - \bG_1 - \bG_1^\sT & \bB_{12} - \bG_2 \\
\bB_{21} - \bG_2^\sT & \bB_{22}
\end{pmatrix},
\]
where $\bB_{ij} = \bP_i^\sT \bB \bP_j $ for $i,j = 1,2$. We readily deduce that 
\[
\min_{\bG \in \R^{r \times d}} 2 \| \bB - \bP_1 \bG - \bG^\sT \bP_1^\sT \|_F^2 = 2 \| \bP_2^\sT \bB \bP_2 \|^2_{F}.
\]
Let us compute its expectation over $\bw_i \sim \normal ( \bzero , \id_d )$, i.e over $\bP_2 = [ \bv_1 , \ldots , \bv_{d-N} ]$ where the $\bv_i \in \R^d$ are $(d-N)$ orthogonal vectors uniformly distributed on the unit sphere in $\R^d$. Let $\bB = \sum_{i = 1}^s \lambda_i \be_i \be_i^\sT$ with $\be_i$ the orthonormal eigenvectors of $\bB$. We get:
\begin{equation}
\begin{aligned}
\E [\| \bP_2^\sT \bB \bP_2 \|^2_{F} ] = & \sum_{i,j = 1}^s \sum_{k,l = 1}^{d-N} \lambda_i \lambda_j \E [\< \bv_k , \be_i \> \< \bv_k, \be_j \> \< \bv_l , \be_i \> \< \bv_l , \be_j \> ] \\
=& \| \bB \|^2_F (d - N) \E [ \< \bv_1 , \be_1 \>^4] + \| \bB \|^2_F (d -N ) (d - N - 1) \E [ \< \bv_1 , \be_1 \>^2 \< \bv_2 , \be_1 \>^2 ] \\
 & + 2 \Big( \sum_{i <j } \lambda_i \lambda_j \Big)  (d - N) \E [ \< \bv_1 , \be_1 \>^2 \< \bv_1 , \be_2 \>^2 ]  \\
 & +  2 \Big( \sum_{i <j } \lambda_i \lambda_j \Big) (d -N ) (d - N-1) \E [ \< \bv_1 , \be_1 \> \< \bv_1 , \be_2 \> \< \bv_2 , \be_1 \> \< \bv_2 , \be_2 \> ]   \Big].
 \label{eq:expansion_NT_QF} 
\end{aligned} 
\end{equation}
We bound each term separately. For $\bu \sim \text{Unif} ( \S^{d-1})$, we have the convergence in distribution of the first two coordinates $\sqrt{d} (u_1 , u_2 ) \Rightarrow \normal ( \bzero , \id_2 ) $, hence:
\begin{equation}
\lim_{d \rightarrow \infty} d^2 \E [ \< \bv_1 , \be_1 \>^4] = 3, \qquad \lim_{d \rightarrow \infty} d^2 \E [ \< \bv_1 , \be_1 \>^2 \< \bv_1 , \be_2 \>^2 ] = 1.
\label{eq:terms_NT_QF}
\end{equation}
Furthermore, conditioned on $\bv_1$, $\bv_2$ is uniformly distributed over the sphere $\S^{d-2}$ in the hyperplane orthogonal to $\bv_1$. We get the uniform convergence
\[
\begin{aligned}
\lim_{d \rightarrow \infty} \sup_{\bv_1 \in \S^{d-1}} | d \E [ \< \bv_2 , \be_1 \>^2 | \bv_1] - (1 - \< \bv_1 , \be_1 \>^2) | = 0.
\end{aligned}
\]
By dominated convergence theorem, we get
\begin{equation}
\lim_{d \rightarrow \infty} d^2 \E [ \< \bv_1 , \be_1 \>^2 \< \bv_2 , \be_1 \>^2 ] = 1.
\label{eq:term3_NT_QF}
\end{equation}
The last term of the sum \eqref{eq:expansion_NT_QF} is also derived by first conditioning on $\bv_1$. Let us denote $\bz_1 = \bP_{\perp \bv_1} \be_1$ and $\bz_2 =  \bP_{\perp \bv_1} \be_2$ the projections of $(\be_1 , \be_2)$ on the hyperplane perpendicular to $\bv_1$, on which $\bv_2$ is uniformly distributed over the unit sphere. We decompose $\bz_2$ into two components: one along $\bz_1$ that we denote  $\bz_2^{(1)} = \bP_{\parallel \bz_1} \bz_2$ and one perpendicular to $\bz_1$, denoted $\bz_2^{(2)} = \bP_{\perp \bz_1}  \bz_2$. Then we have:
\[
\begin{aligned}
 \E [ \< \bv_2 , \be_1 \> \< \bv_2 , \be_2 \> | \bv_1] =&  \E [ \< \bv_2 , \bz_1 \> \< \bv_2 , \bz_2 \> | \bv_1] \\
 = & \E \Big[ \< \bv_2 , \bz_1 \> \Big(  \< \bv_2 , \bz_2^{(1)} \> + \< \bv_2 , \bz_2^{(2)} \> \Big) \Big\vert \bv_i \Big] \\
 =& \< \bz_1 , \bz_2\> \E [ u_1^2 ] + \| \bz_1 \|_2 \|\bz_2^{(2)} \|_2 \E [ u_1 u_2 ] \\
 =& \frac{ \< \bz_1 , \bz_2\>  }{d - 1},
\end{aligned}
\]
where $(u_1 , u_2)$ are the first two coordinates of a uniform random variable on the sphere $\S^{d-2}$. Using that:
\[
\< \bz_1 , \bz_2\>  = \< \be_1 - \< \be_1 , \bv_1 \> \bv_1 , \be_2 - \< \be_2 , \bv_1 \> \bv_1 \> = -  \< \be_1 , \bv_1 \> \< \be_2 , \bv_1 \>,
\]
we get
\begin{equation}
\E [ \< \bv_1 , \be_1 \> \< \bv_1 , \be_2 \> \< \bv_2 , \be_1 \> \< \bv_2 , \be_2 \> ] = - \frac{1}{d-1} \E [ \< \bv_1 , \be_1 \>^2 \< \bv_1 , \be_2 \>^2] = - \frac{1}{d^3 } + o_d (d^{-3}),
\label{eq:term4_NT_QF}
\end{equation}
where we used the same argument as for \eqref{eq:terms_NT_QF}. Plugging the above limits \eqref{eq:terms_NT_QF}, \eqref{eq:term3_NT_QF} and \eqref{eq:term4_NT_QF} in the expansion \eqref{eq:expansion_NT_QF}, we get
\begin{equation}
\E [ R_{\NT,N} (f_*)]  = 2 \| \bB \|_F^2 \Big[ (1-\rho)_+^2 + (1 - \rho)_+ \frac{ \tr(\bB)^2}{d \|\bB\|_F^2 }  - (1-\rho)_+^2 \frac{\tr(\bB)^2}{d \|\bB\|_F^2} + o_d (1)\Big].
\label{eq:NT_risk_QF}
\end{equation}
Recalling the assumption $\E ( f_* ) = 0$, we have $\|f_* \|^2_{L^2}  = 2  \| \bB \|_F^2 $, which concludes the proof.
\end{proof}

\begin{remark}
The above formula for the $\RF$ risk Eq.~\eqref{eq:NT_risk_QF} has two terms that corresponds to the two limits $\tr(\bB) /\|\bB\|_F = o_d (\sqrt{d}) $ (e.g. spiked matrix)
\[
\E [ R_{\NT,N} (f_*)] = 2 (1-\rho)_+^2 \| \bB \|_F^2 + o_d (\| \bB \|_F^2) ,
\]
and $\tr(\bB)^2 = d \|\bB\|_F^2  $ (i.e. $\bB \propto \id $)
\[
\E [ R_{\NT,N} (f_*)] = 2 (1 - \rho)_+ \| \bB \|_F^2 .
\]
\end{remark} 

It is possible to show concentration of $\| \bP_2^\sT \bB \bP_2 \|^2_{F} $ on its mean $\E [ \| \bP_2^\sT \bB \bP_2 \|^2_{F}]$  for $\bB$ that satisfies $\| \bB \|_{\op} \| \bB \|_{F} \leq C$ (see Theorem \ref{thm:NTK_MG}).

\subsection{Neural Network model: proof of Theorem \ref{thm:NN_quadratic}}

We consider two-layers neural networks with quadratic activation function $\sigma(x) = x^2$ and we fix the second layer weights to $1$, 
\[
\hat f(\bx; \bW, c) = \sum_{i=1}^N \< \bw_i, \bx\>^2 + c. 
\]
We consider the ground truth function $f_*$ to be a quadratic function as per Eq.~\eqref{eq:QF_proof}, and the risk function defined by
\[
L(\bW, c) = \E_{\bx}[(f_*(\bx) - \hat f(\bx; \bW, c))^2] =  \E_{\bx}\Big[\Big( \< \bx \bx^\sT, \bB - \bW \bW^\sT \> + b_0 - c \Big)^2\Big]. 
\]
We consider running SGD dynamics upon the risk function for a fresh sample $(\bx_k, f_*(\bx_k))$ for each iteration
\[
(\bW_{k + 1}, c_{k+1}) = (\bW_k, c_k) - \varepsilon \nabla_{\bW, c} \Big( f_*(\bx_k) - \hat f (\bx_k; \bW, c) \Big)^2, 
\]
and denote 
\[
R_{\NN, N}(f_{*};\ell, \varepsilon) = \E_{\bx}[(f_*(\bx) - \hat f(\bx; \bW_\ell, c_\ell))^2 ]. 
\]

\subsubsection{Global minimum}

\begin{lemma}\label{lem:global minimizer_A}
Let $f_* = \< \bx, \bB \bx\> + b_0$ for some $\bB\succeq 0$ and $b_0 \in \R$. Denote by $(\lambda_{i}(\bB))_{i \in [r]}$ the positive eigenvalues of $\bB$ in descending order. Then we have
\[
\inf_{\bW, c} L(\bW, c) = 2 \sum_{i= N + 1}^{r} \lambda_{i}(\bB)^2.
\]
\end{lemma}

\begin{proof}[Proof of Lemma \ref{lem:global minimizer_A}]
Note we have 
\[
\begin{aligned}
L(\bW, c) =& \E_{\bx}[ (\< \bB - \bW \bW^\sT, \bx \bx^\sT\> + b_0 - c)^2 ] \\
=& 2 \| \bB - \bW \bW^\sT \|_F^2  + \tr(\bB - \bW \bW^\sT)^2 - 2 \tr(\bB - \bW \bW^\sT) (c - b_0) + (c - b_0)^2,
\end{aligned}
\]
minimizing over $c$ gives
\[
\inf_c L(\bW, c) = 2 \| \bB - \bW \bW^\sT \|_F^2. 
\]
The infimum of $L$ over $\bW$ is equivalent to the low-rank approximation problem of matrix $\bB$ in Frobenius norm, with rank less or equal to $\max(d,N)$, and  is given by the Eckart-Young-Mirsky theorem (see \cite{eckart1936approximation}). 
\end{proof}

\subsubsection{Landscape: proof of Proposition \ref{prop:landscape_NN_QF}}

Without loss of generality, throughout the proof, we assume that $\bB$ is diagonal and $b_0 = 0$. Our first proposition characterizes the critical points of $L(\bW, c)$. 

\begin{proposition}
Let $\bW \in \R^{d \times N}$, and $\bB \in \R^{d \times d}$ to be a positive semi-definite diagonal matrix. Define the risk function to be
\[
L(\bW , c ) = \E_{\bx} [ ( \< \bB - \bW \bW^\sT, \bx \bx^\sT \> - c)^2 ]. 
\]
Then for any critical point $(\bW_0,c_0)$ of $L(\bW,c)$, there exists a projection matrix $\bP = \sum_{i=1}^k \be_{\tau(i)} \be_{\tau(i)}^\sT$ for some injection $\tau: [k] \to [d]$, such that $\bGamma_0 = \bW_0 \bW_0^\sT$ is diagonal and satisfy
\[
\begin{aligned}
\bGamma_0 =& \bP \bB \bP, \\
c_0 =& \tr(\bB - \bGamma_0).
\end{aligned}
\]
\end{proposition}

\begin{proof}
Calculating the risk function, we get
\[
\begin{aligned}
L ( \bW , c ) = c^2 + 2c \cdot \tr ( \bW \bW^\sT - \bB ) + \tr ( \bW \bW^\sT - \bB)^2 + 2 \| \bW \bW^\sT - \bB \|^2_F.
\end{aligned}
\]
We consider the gradient of this function. We get:
\[
\begin{aligned}
\frac{\partial }{\partial c} L ( \bW , c ) & = 2c + 2 \tr ( \bW \bW^\sT - \bB ), \\
\nabla_{\bW} L ( \bW , c ) &= 2c \bW + 2 \tr ( \bW \bW^\sT - \bB ) \bW + 8 ( \bW \bW^\sT - \bB) \bW.
\end{aligned}
\]
By the stationary condition, at a critical point $(\bW_0 , c_0)$, we must have:
\begin{align}
c_0 & = - \tr ( \bW_0 \bW_0^\sT - \bB ), \label{eq:critical_c_equ} \\
\bB \bW_0 & = \bW_0 \bW_0^\sT \bW_0. \label{eq:critical_condition}
\end{align}
Let us denote $\bW_0 = \bU \bS \bV^\sT$ the (extended) singular value decomposition of $\bW_0 \in \R^{d \times N}$ with $\bU \in \R^{d \times d}$, $\bS \in \R^{d \times N}$ and $\bV \in \R^{N\times N}$. Then the stationary condition \eqref{eq:critical_condition} gives
\begin{equation}\label{eqn:crit}
\bB \bU  \bS \bV^\sT = \bU \bS^3 \bV^\sT.
\end{equation}
Let $r$ be the rank of $\bW_0$ and $\bS = \diag(\bS_1 , \bzero)$, $\bU = (\bU_1 , \bU_2)$ with $\bS_1\in \R^{r \times r}$, $\bU_1 \in \R^{d \times r}$ and $\bU_2 \in \R^{d \times (d-r)}$. Then we get:
\[
\bB \bU_1 = \bU_1 \bS_1^2. 
\]
This is of the form of the eigenvalue equation of matrix $\bB$. Hence we must have the columns of $\bU_1$ to be a set of eigenvectors and $\bS_1^2$ to be positive eigenvalues of $\bB$. This proves the proposition. 
\end{proof}

Note the global minimizers are attained for $\bGamma_0 = \bW_0 \bW_0^\sT$ corresponding to the $\min (N, d )$ directions of $\bB$ with the largest eigenvalues. We prove in the following proposition that stationary points that are not global minimizers are strict saddle points.

Define the spectral separation of $\bB$ as
\[
\delta^{\textup{sep}} = \min \{ \vert \lambda_i ( \bB) - \lambda_j (\bB) \vert \, : \, i,j \in [d], \lambda_i ( \bB )  \neq \lambda_j ( \bB ) \},
\]
and $\delta^{\textup{eig}}$ the minimum strictly positive eigenvalue of $\bB$.

\begin{proposition}\label{eqn:NN_quadratic_strict_saddle}
Consider $(\bW_0 , c_0 )$ a stationary point of $L(\bW , c)$ but not a global minimizer. Then, we have
\[
\lambda_{\min} ( \nabla^2_{\bW} L ( \bW_0 , c_0 ) ) \leq - 4 \min \{ \delta^{\textup{eig}} , \delta^{\textup{sep}} \} < 0.
\]
\end{proposition}

\begin{proof}
Let us first compute the Hessian of the risk with respect to the $\bW$ variable. We have
\[
\begin{aligned}
\< \bZ , \nabla^2_{\bW} L( \bW , c) \bZ \>  =& 2c \cdot \tr ( \bZ \bZ^\sT ) + 2 \tr (\bW \bW^\sT - \bB )  \tr ( \bZ \bZ^\sT ) +  4 \tr (\bW \bZ^\sT )^2 \\
& + 4 \| \bW \bZ^\sT \|^2_F + 4 \tr ( \bW \bZ^\sT \bW \bZ^\sT ) + 4 \< \bW\bW^\sT - \bB , \bZ \bZ^\sT \>.
\end{aligned}
\]
Plugging the value of $c_0$ at a critical point (cf Eq.~\eqref{eq:critical_c_equ}), we get
\begin{equation}\label{eqn:Hessian_of_risk}
\begin{aligned}
\< \bZ , \nabla^2_{\bW} L( \bW_0 , c_0) \bZ \>  =& 4 \tr (\bW_0 \bZ^\sT )^2  + 4 \| \bW_0 \bZ^\sT \|^2_F + 4 \tr ( \bW_0 \bZ^\sT \bW_0 \bZ^\sT ) + 4 \< \bW_0\bW_0^\sT - \bB , \bZ \bZ^\sT \>.
\end{aligned}
\end{equation}
\noindent
{\bf Case 1:} Consider the case $\text{rank}(\bW_0) < \min \{\text{rank}(\bB) , N\}$. Then there exists an $i \in [d]$ such that $\bB_{ii} > 0$ (recall that we assumed $\bB$ diagonal , with diagonal elements given by the positive eigenvalues of $\bB$) and $(\bW_0 \bW_0^\sT )_{ii} = 0$. For simplicity, let us permute the coordinates so that $i =1$. The singular value decomposition of $\bW_0$ verifies
\[
\bW_0 = \bU_0 \bS_0 \bV_0^\sT = \begin{pmatrix} 
0 & 0 & \ldots & 0  \\
0 & & & \\
\vdots & & \Tilde{\bU_0} \Tilde{\bS_0} & \\
0 & & & 
\end{pmatrix} \bV_0^\sT ,
\]
where $\Tilde{\bU_0}$ and $ \Tilde{\bS_0}$ are the sub-matrices corresponding respectively to the $(d-1) \times (d-1)$ last coordinates of $\bU_0$ and $(d-1) \times (N-1)$ last coordinates of $\bS_0$. Let us consider 
\[
\bZ = \begin{pmatrix} 
1 & 0 & \ldots & 0  \\
0 & & & \\
\vdots & & \bzero & \\
0 & & & 
\end{pmatrix} \bV_0^\sT.
\]
We have $\| \bZ \|_F = 1$ and $\bW_0 \bZ^\sT = 0$. Plugging these matrices in the above expression of the Hessian, see Eq.~\eqref{eqn:Hessian_of_risk}, we get
\[
\< \bZ , \nabla^2_{\bW} L( \bW_0 , c_0) \bZ \>  = - 4  \bB_{11} \le - 4 \delta_{\textup{eig}}.
\]
\noindent
{\bf Case 2:} Consider the case when $\text{rank}(\bW_0 \bW_0^\sT) = N < \text{rank}(\bB) $ and $\bW_0 \bW_0^\sT$ does not correspond to the $N$ largest eigenvalues of $\bB$. Then there exists $i \neq j \in [n]$, such that $\bB_{ii} > \bB_{jj}$, $(\bW_0 \bW_0^\sT )_{ii} = 0$ and $(\bW_0 \bW_0^\sT )_{jj} = \bB_{jj}$. For simplicity, let us permute the coordinates such that $i =1$ and $j=2$. The SVD decomposition of $\bW_0$ now verifies:
\[
\bW_0 = \bU_0 \bS_0 \bV_0^\sT = \begin{pmatrix} 
0 & 0 & \ldots & 0  \\
\sqrt{\bB_{22} }  & 0 & \ldots & 0  \\
0 & & & \\
\vdots & & \Tilde{\bU_0} \Tilde{\bS_0} & \\
0 & & & 
\end{pmatrix} \bV_0^\sT ,
\]
where $\Tilde{\bU_0} \Tilde{\bS_0}$ is the sub-matrix  of the last $(d-2) \times (N-1)$ coordinate of $\bU_0 \bS_0$. Let us consider again
\[
\bZ = \begin{pmatrix} 
1 & 0 & \ldots & 0  \\
0 & & & \\
\vdots & & \bzero & \\
0 & & & 
\end{pmatrix} \bV_0^\sT.
\]
We have $\| \bZ \|_F = 1$. Plugging these matrices in the above expression of the Hessian (\ref{eqn:Hessian_of_risk}), note
\[
\tr( \bW_0 \bZ^\sT) =  \tr ( \bW_0 \bZ^\sT \bW_0 \bZ^\sT ) = 0,~~~~ \| \bW_0 \bZ^\sT \|_F^2 = \bB_{22}, ~~~~ \< \bW_0\bW_0^\sT - \bB , \bZ \bZ^\sT \> = \bB_{11},
\]
we get
\[
\< \bZ , \nabla^2_{\bW} L( \bW_0 , c_0) \bZ \>  = - 4  ( \bB_{11} - \bB_{22}) \leq -4 \delta^{\textup{sep}}.
\]
This proves the proposition. 
\end{proof}


We can now prove Proposition \ref{prop:landscape_NN_QF}.
\begin{proof}[Proof of Proposition \ref{prop:landscape_NN_QF}]
First, remark that $L (\bW , c)$ has compact sub-level sets. The proposition then follows from Proposition \ref{eqn:NN_quadratic_strict_saddle} and the continuity of the gradient $\nabla L (\bx)$ and of the minimum eigenvalue of the Hessian $\lambda_{\min} (\nabla^2 L(\bx) )$.
\end{proof}


\subsubsection{Dynamics}

The following lemma is a standard combination of Lojasiewicz inequality and center and stable manifold theorem. We prove it for completeness. 

\begin{lemma}\label{lem:center_stable_analytic}
Let $f:\R^d \to \R$ be an analytic function that has compact level sets. Consider the gradient flow 
\[
\dot \bx_t = - \nabla f(\bx_t). 
\]
Then for (Lebesgue) almost all initialization $\bx_0$, there exists a second order local minimizer $\bx_*$, such that
\[
\lim_{t \to +\infty} \bx_t = \bx_*. 
\]
\end{lemma}

\begin{proof}[Proof of Lemma \ref{lem:center_stable_analytic}]~

\noindent
{\bf Step 1. Show convergence to a critical point. } Since $f$ is an analytic function, by Lojasiewicz inequality \cite{lojasiewicz1982trajectoires}, and the fact that the level set of $f$ is compact, we have 
\[
\lim_{t \to +\infty } \bx_t = \bx_*
\]
for $\bx_*$ some critical point of $f$. 

\noindent
{\bf Step 2. Show convergence to a local minimizer. } In this step, we proceed similarly to the proof of Theorem 3 in \cite{panageas2016gradient}. First, consider a sublevel set 
\[
\Omega(K) = \{ \bx: f(\bx) \le K\}. 
\]
Then we have $\Omega(K)$ compact. Since $f$ is an analytic function, $\nabla f$ is Lipschitz in the compact set $\Omega(K)$. We define the map $\phi_t: \Omega(K) \to \phi_t(\Omega(K))$, $ \bx \mapsto \bx_t$ where $\bx_t$ is defined as the solution of 
\[
\begin{aligned}
\dot \bx_t =& -\nabla f(\bx_t),\\
\bx_0 =& \bx.
\end{aligned}
\]
By Picard's existence and uniqueness theorem, we have $\phi_t$ is a diffeomorphism from $\Omega(K)$ to $\phi(\Omega(K))$ for any $t > 0$. Fix an $\varepsilon_0 > 0$, and we define $g = \phi_{\varepsilon_0}: \Omega(K) \to \Omega(K)$.

Let $\br$ be a strict saddle point of $f$, then $\br$ must be an unstable fixed point of the diffeomorphism $g = \phi_{\eps_0}$. By center and stable manifold theorem (such as Theorem 9 in \cite{panageas2016gradient}), there exists a manifold $W_{\loc}^{\sc}(\br)$ of dimension at most $d-1$, and a ball $\ball(\br, \varepsilon(\br))$ centered at $\br$ with radius $\varepsilon(\br)$, such that we have the following facts: 
\begin{enumerate}
    \item[(1)] $g \left(W_{\loc}^{\sc}(\br) \cap \ball(\br, \varepsilon(\br)) \right) \subseteq W_{\loc}^{\sc}(\br)$;
    \item[(2)] If $g^n(\bx) \in \ball(\br, \varepsilon(\br))$ for all $n \ge 0$, we have $\bx \in W_{\loc}^{\sc}(\br)$ (here $g^n$ means composition of $g$ for $n$ times).
\end{enumerate}

We consider the union of the balls associated to all the strict saddle points of $f$ in $\Omega(K)$
\[
A = \cup_{\br \in \Omega(K): \br \text{ strict saddle}} \ball(\br, \varepsilon(\br)). 
\]
Due to Lindelof's lemma, we can find a countable subcover for $A$, i.e., there exists
fixed-points $\br_1, \br_2, \ldots$ such that $A = \cup_{m = 1}^\infty \ball(\br_m, \varepsilon(\br_m))$. If gradient descent converges to a strict saddle point, starting from a point $\bv \in \Omega(K)$, there must exist a $t_0$ and $m$ such that $\phi_t (\bv) \in \ball(\br_m, \varepsilon(\br_m))$ for all $t \ge t_0$.
By center and stable manifold theorem, we get that $\phi_t (\bv) \in W_{\loc}^{\sc}(\br_m) \cap \Omega(K)$. By setting $D_1(\br_m) = g^{-1}(W_{\loc}^{\sc} (\br_m) \cap \Omega(K))$ and $D_{i+1}(\br_m) = g^{-1}(D_i(\br_m) \cap \Omega(K))$ we get that $\bv \in D_k (\br_m)$ for all $k \varepsilon_0 \ge t_0$. Hence the set of initial points in $\Omega(K)$ such that gradient descent converges to a strict saddle point is a subset of
\[
P = \cup_{m = 1}^\infty \cup_{k \in \mathbb N} D_k(\br_m).
\]
Note that the set $W_{\loc}^{\sc} (\br_m) \cap \Omega(K)$ has Lebesgue measure zero in $\R^d$. Since $g$ is a diffeomorphism, $g^{-1}$ is continuously differentiable and thus it is locally Lipschitz. Therefore, $g^{-1}$ preserves the null-sets and hence (by induction) $D_i(\br_m)$ has measure zero for all $i$. Thereby we get that $P$ is a countable union of measure zero sets. Hence $P$ has measure $0$. 

Finally, note we have 
\[
\{\bx \in \Omega(K): \exists \br, \br \text{ is strict saddle}, \br =  \lim_{t \to +\infty} \phi_t(\bx)\} \subseteq P. 
\]
Since $P$ has measure $0$, we have 
\[
\begin{aligned}
&\{\bx \in \R^d: \exists \br, \br \text{ is strict saddle}, \br =  \lim_{t \to +\infty} \phi_t(\bx)\}\\ =& \cup_{K \in \mathbb N}\{\bx \in \Omega(K): \exists \br, \br \text{ is strict saddle}, \br =  \lim_{t \to +\infty} \phi_t(\bx)\}
\end{aligned}
\]
has measure $0$. This proves the lemma. 
\end{proof}

The following lemma is standard, and a corollary of Theorem 2.11 in \cite{kurtz1970solutions}. 

\begin{lemma}\label{lem:Markov_iterates}
Let 
\[
F(\bx) = \E_{\bz}[f(\bx; \bz)]
\]
be a $C^2$ function on $\Omega \subseteq \R^d$. Assume 
\[
\begin{aligned}
\sup_{\bx \in \Omega}\E_\bz[\| \nabla_\bx f(\bx; \bz) \|_2] <& \infty,\\
\sup_{\bx \in \Omega} \| \nabla^2 F(\bx) \|_{\rm op} <& \infty.
\end{aligned}
\]
Let $\bx_t$ be the trajectory of 
\[
\begin{aligned}
\dot \bx_t =& - \nabla F(\bx_t), \\
\end{aligned}
\]
with initialization $\bx_0 \in \Omega$. Further assume that there exists $\eta > 0$, such that $\cup_{t \ge 0} \ball(\bx_t, \eta) \subseteq \Omega$. 

Consider the following Markov jump process $ \bx_{t, \varepsilon}$ starting from $\bx_0$, with jump time to be an exponential random variable with fixed mean $\varepsilon$, and jump direction $-\varepsilon \nabla f(\bx; \bz)$ where $\bx$ is the current state, and $\bz$ an independent sample. Then we have for any fixed $T > 0$ and $\delta > 0$, 
\[
\lim_{\varepsilon \to 0+} \P\Big( \sup_{0 \le t \le T} \| \bx_t - \bx_{t, \varepsilon} \|_2 \ge \delta \Big) = 0.
\]
\end{lemma}

\subsubsection{Proof of Theorem \ref{thm:NN_quadratic}}

By Proposition \ref{eqn:NN_quadratic_strict_saddle}, we know that for $L(\bW, c)$, any critical point that is not a global minimizer is a strict saddle point. Consider the gradient flow
\[
\frac{\de}{\de t}(\bW_t, c_t)  = - \nabla L(\bW_t, c_t)
\]
with random initialization $(\bW_0, c_0) \sim \nu_0$ where $\nu_0$ is a distribution that is absolutely continuous with respect to Lebesgue measure.  Since $L(\bW, c)$ is an analytic function, by Lemma \ref{lem:center_stable_analytic}, we have $(\bW_t, c_t)$ converges to a global minimizer of $L(\bW, c)$. That is, we have almost surely (over $\nu_0$)
\[
\lim_{t \to \infty} L(\bW_t, c_t) = \inf_{\bW, c} L(\bW, c), 
\]
where $\inf_{\bW, c} L(\bW, c)$ is calculated in Lemma \ref{lem:global minimizer_A}. 

Consider the following Markov jump process $ (\bW_{t, \varepsilon}, c_{t, \varepsilon})$ starting from $(\bW_0, c_t) \sim \nu_0$, with jump time to be an exponential random variable with fixed mean $\varepsilon$, and jump direction to be $-\varepsilon \nabla L(\bW, c; \bz)$ where
\[
\nabla L(\bW, c; \bz) = \begin{pmatrix} \nabla_{\bW} L(\bW, c; \bz)  \\ \partial_c L(\bW, c; \bz) \end{pmatrix} = \begin{pmatrix} 2 ( c - b_0 + \< \bz \bz^\sT,  \bW \bW^\sT - \bB\>) \bz \bz^\sT \bW \\ 2 ( c - b_0 + \< \bz \bz^\sT,  \bW \bW^\sT - \bB\>) \end{pmatrix} 
\]
with $(\bW, c)$ the current state, and $\bz$ an independent sample. By Lemma \ref{lem:Markov_iterates}, we have for any fixed $T > 0$ and $\delta > 0$, 
\[
\lim_{\varepsilon \to 0+} \P\Big( \sup_{0 \le t \le T} \| (\bW_{t, \eps}, c_{t, \eps})- (\bW_{t}, c_t) \|_2 \ge \delta \Big) = 0.
\]
Note the sequence of Markov jump process at jump time is exactly the SGD iterates. Hence the SGD iterates with properly scaled number of iterations is uniformly close to $(\bW_t, c_t)$ over finite horizon as $\eps \to 0$. This proves the Theorem.

\section{Proofs for Mixture of Gaussians}

In this section, we consider the mixture of Gaussian setting (\MG):  $y_i = \pm 1$ with equal probability $1/2$, and $\bx_i | y_i =+1 \sim \normal ( 0 , \bSigma^{(1)})$, $\bx_i | y_i = -1 \sim \normal ( 0 , \bSigma^{(2)})$ where $\bSigma^{(1)} = \bSigma - \bDelta$ and $\bSigma^{(2)} = \bSigma + \bDelta$. With these notations, 
\[
\begin{aligned}
\bSigma = &\frac{1}{2} ( \bSigma^{(1)} + \bSigma^{(2)}), \\
\bDelta = &\frac{1}{2} ( \bSigma^{(2)} - \bSigma^{(1)}).
\end{aligned}
\]
Throughout this section, we will make the following assumptions:
\begin{itemize}
\item[{\bf M1.}] There exists constants $0 < c_1 < c_2$ such that $c_1 \id_d \preceq \bSigma \preceq c_2 \id_d$;
\item[{\bf M2.}]  $\| \bDelta \|_{\rm op} = \Theta_d(1/\sqrt{d})$.
\end{itemize}
Throughout this section, we will denote $\P_{\bSigma,\bDelta}$ the joint distribution of $(y, \bx )$ under the \MG \, model, $\E_{\bx,y}$ the expectation operator with respect to $(y, \bx ) \sim \P_{\bSigma,\bDelta}$ and $\E_{\bx}$ the expectation operator with respect to the marginal distribution $\bx \sim (1/2) \cdot \normal ( 0 , \bSigma^{(1)}) + (1/2) \cdot \normal ( 0 , \bSigma^{(2)})$.

\subsection{Random Features model: proof of Theorem \ref{thm:RFMixture}}

Recall the definition
\[
R_{\RF, N}(\P) = \arg\min_{\hf\in\cF_{\RF,N}(\bW)}\E\big\{(y -\hf(\bx))^2\big\},
\]
where
\[
\cF_{\RF, N}(\bW) = \Big\{ f_N( \bx) =  \sum_{i=1}^N  a_i\sigma(\< \bw_i, \bx\>): \; a_i \in \R, i \in [N] \Big\}.
\]
Note that it is easy to see from the proof that the result stays the same if we add an offset $c$.

\begin{remark}
We will state the lemmas for the case $\bSigma = \id_d$, which amounts to re-scaling $\Tilde \bGamma = \bSigma^{1/2} \bGamma \bSigma^{1/2}$ and $\Tilde \bDelta = \bSigma^{-1/2} \bDelta \bSigma^{-1/2}$.
\end{remark}

\subsubsection{Representation of the $\RF$ risk}
\begin{lemma}\label{lem:GM_risk_representation}
Consider the RF model introduced above. We have
\begin{align}
R_{\RF,N}(\P_{\id, \bDelta}) = \E_{\bx, y}[y ^2] - \bV^\sT \bU^{-1} \bV, 
\end{align}
where $\bV = [V_1, \ldots, V_N]^\sT$, and $\bU = (U_{ij})_{i, j \in [N]}$, with
\[
\begin{aligned}
V_i =& \E_{\bx, y}[y \sigma(\< \bw_i, \bx\>)], \\
U_{ij} =& \E_{\bx, y}[\sigma(\< \bw_i, \bx\>) \sigma(\< \bw_j, \bx\>)].
\end{aligned}
\]
\end{lemma}
\begin{proof}
Simply write the KKT conditions. The optimum is achieved at $\ba = \bU^{-1} \bV$. 
\end{proof}

\subsubsection{Approximation of kernel matrix $\bU$}

\begin{lemma}\label{lem:RF_GM_kernel}
Let $\sigma \in L^2(\normal(0, 1))$ be an activation function. Denote $\lambda_k =\E_{G \sim \normal(0, 1)}[\sigma(G) \He_k(G)]$ the $k$-th Hermite coefficient of $\sigma$ and assume $\lambda_0 = 0$. 
Let $\bU = (U_{ij})_{i, j \in [N]}$ be a random matrix with
\[
\begin{aligned}
U_{ij} =& \E_{\bx  }[\sigma(\< \bw_i, \bx\>) \sigma(\< \bw_j, \bx\>)],
\end{aligned}
\]
where $(\bw_i)_{i \in [N]} \sim \normal(\bzero, \bGamma)$ independently. Assume conditions {\bf A1} and {\bf B2} hold. 

Define $\bW = (\bw_1, \ldots, \bw_N) \in \R^{d \times N}$, and $\bU_0 = \{(U_0)_{ij}\}_{i,j\in[N]}$, with
\[
(U_0)_{ij} = \tilde \lambda \delta_{ij} + \lambda_1^2 \< \bw_i, \bw_j\> + \kappa / d + \mu_{i} \mu_{j},
\]
where 
\[
\begin{aligned}
\mu_{i} =& \lambda_2 ( \| \bw_i \|_2^2 - 1) / 2,\\
\tilde \lambda =& \E[\sigma(G)^2] - \lambda^2_1, \\
\kappa =&  d \cdot \lambda_2^2 [\tr(\bGamma^2)/2 + \tr(\bDelta \bGamma)^2 / 4].
\end{aligned}
\]
Then we have as $N/d = \rho$ and $d \to \infty$, we have
\[
\| \bU - \bU_0 \|_{\op} = o_{d, \P}(1).
\]
\end{lemma}

\begin{proof}[Proof of Lemma \ref{lem:RF_GM_kernel}]
Recalling that in the ($\MG$) model, we have $\bx \sim (1/2) \cdot \normal(\bzero, \id - \bDelta) + (1/2) \cdot \normal(\bzero, \id + \bDelta)$, we have
\[
\begin{aligned}
U_{ij} =& \E_{\bx  }[\sigma(\< \bw_i, \bx\>) \sigma(\< \bw_j, \bx\>)] \\
= &\Big\{  \E_{\bx \sim \normal(\bzero, \id) } [\sigma(\< (\id - \bDelta)^{1/2} \bw_i, \bx\>) \sigma(\< (\id - \bDelta)^{1/2} \bw_j, \bx\>)] \\
& +  \E_{\bx \sim \normal(\bzero, \id) } [\sigma(\< (\id + \bDelta)^{1/2} \bw_i, \bx\>) \sigma(\< (\id + \bDelta)^{1/2} \bw_j, \bx\>)] \Big\}/2.
\end{aligned}
\]
We can therefore readily use the result of Lemma \ref{lem:RF_QF_kernel} for $\Tilde \bw_i \sim \normal ( \bzero , (\id - \bDelta)^{1/2} \bGamma (\id - \bDelta)^{1/2} )$ and $\Tilde \bw_i \sim \normal ( \bzero , (\id + \bDelta)^{1/2} \bGamma (\id + \bDelta)^{1/2} )$, to get
\begin{equation}
\| \bU - \tilde \bU_0 \|_{\op} = o_{d, \P}(1),
\label{eq:U_bound_1_MG}
\end{equation}
where $\tilde \bU_0 = (\tilde U_0)_{i, j \in [N]}$ with 
\[
(\tilde U_0)_{ij} = \tilde \lambda \delta_{ij} + \lambda_1^2 \< \bw_i, \bw_j\> + \kappa / d + (\mu_{i}^+ \mu_{j}^+ + \mu_{i}^- \mu_{j}^-)/2,
\]
and
\[
\begin{aligned}
\tilde \lambda =& \E[\sigma(G)^2] - \lambda^2_1, \\
\Tilde \kappa =&  d \lambda_2^2 [\tr((\id - \bDelta) \bGamma (\id - \bDelta) \bGamma )+ \tr((\id + \bDelta) \bGamma (\id + \bDelta) \bGamma) ] / 4 \\
=&  d \lambda_2^2 [\tr( \bGamma^2 )+ \tr(\bDelta \bGamma \bDelta \bGamma) ] / 2,
\\
\mu_{i}^+ =& \lambda_2 ( \| (\id + \bDelta)^{1/2} \bw_i \|_2^2 - 1) / 2,\\
\mu_{i}^- =& \lambda_2 ( \| (\id - \bDelta)^{1/2} \bw_i \|_2^2 - 1) / 2.
\end{aligned}
\]
Note that we have
\[
(\mu_{i}^+ \mu_{j}^+ + \mu_{i}^- \mu_{j}^-)/2 = \mu_{i} \mu_{j} + \lambda_2^2 \< \bw_i , \bDelta \bw_i \> \< \bw_j , \bDelta \bw_j \> /4,
\]
where
\[
\mu_{i} = \lambda_2 ( \| \bw_i \|_2^2 - 1) / 2.
\]
The matrix $(\< \bw_i , \bDelta \bw_i \> \< \bw_j , \bDelta \bw_j \>)_{i,j \in [N]} $ is simply $\bs \bs^\sT$ with $\bs = ( \< \bw_i , \bDelta \bw_i \>)_{i \in [N]}$. Defining $\nu = \E [\< \bw_i , \bDelta \bw_i \> ] = \tr(\bGamma \bDelta)$, we have
\[
\bs \bs^\sT = (\bs - \nu \ones ) \nu \ones^\sT + \nu \ones (\bs - \nu \ones )^\sT + \nu^2 \ones \ones^\sT + (\bs - \nu \ones )(\bs - \nu \ones )^\sT.
\]
Furthermore:
\[
\| \bs - \nu \ones \|_2^2 = \sum_{i=1}^d \tr ( (\bw_i \bw_i^\sT - \bGamma) \bDelta )^2.
\]
Note that by assumptions {\bf M2} and {\bf B2}, we have $\E [\tr ( (\bw_i \bw_i^\sT - \bGamma) \bDelta )^2] = 2\|\bDelta \bGamma \|_F^2 = o_{d,\P} (d^{-1})$. We deduce that $\| \bs - \nu \ones \|_2 = o_{d,\P} (1)$, and therefore
\[
\begin{aligned}
\| (\bs - \nu \ones ) \nu \ones^\sT \|_{\op} = o_{d,\P} (1), \\
\| (\bs - \nu \ones )(\bs - \nu \ones )^\sT \|_{\op} = o_{d,\P} (1).
\end{aligned}
\]
Hence, we get 
\begin{equation}
\| ( \bmu^+ \bmu^{+\sT} + \bmu^- \bmu^{-\sT} )/2  - \bmu \bmu^{\sT} - \tr(\bGamma \bDelta)^2 \ones \ones^\sT \|_{\op} = o_{d,\P} (1).
\label{eq:kernel_obound_1_MG}
\end{equation}
We also have $\tr(\bDelta \bGamma \bDelta \bGamma)^2 = o_d (d^{-1})$ by assumptions {\bf M2} and {\bf B2}, hence 
\begin{equation}
\| \tr(\bDelta \bGamma \bDelta \bGamma) \ones \ones^\sT \|_{\op} = o_{d,\P} (1).
\label{eq:kernel_obound_2_MG}
\end{equation}
Therefore, combining \eqref{eq:kernel_obound_1_MG} and \eqref{eq:kernel_obound_2_MG}, we get:
\begin{equation}
\| \Tilde \bU_0 - \bU_0 \|_{\op} = o_{d,\P} (1).
\label{eq:U_bound_2_MG}
\end{equation}
Combining \eqref{eq:U_bound_1_MG} and \eqref{eq:U_bound_2_MG} concludes the proof.
\end{proof}

\subsubsection{Approximation of the $\bV$ vector}

\begin{lemma}\label{lem:GM_V}
Under the assumption of Theorem \ref{thm:RFMixture}, define $\bV = (V_1, \ldots, V_N)^\sT$ with
\[
V_i = \E_{\bx, y}[y \sigma(\< \bw_i, \bx\>)]
\]
where $(\bw_i)_{i \in [N]} \sim \normal(\bzero, \bGamma)$ independently. Then as $N/d = \rho$ with $d \to \infty$, we have 
\[
\| \bV - \tau  \ones / \sqrt{d} \|_2 =  o_{d, \P}( 1 ) ,
\]
where
\[
\tau = - \sqrt{d} \cdot \lambda_2 \tr(\bDelta \bGamma) / 2.
\]
\end{lemma}

\begin{proof}[Proof of Lemma \ref{lem:GM_V}]
We have 
\[
\begin{aligned}
V_i =& \{ \E_{\bx \sim \normal(\bzero, \id - \bDelta)}[\sigma(\< \bw_i, \bx\>)] - \E_{\bx \sim \normal(\bzero, \id + \bDelta)}[\sigma\< \bw_i, \bx \>]\}/2 \\
=& \{ \E_{\bx \sim \normal(\bzero, \id)}[\sigma(\< (\id - \bDelta)^{1/2}\bw_i, \bx\>)] - \E_{\bx \sim \normal(\bzero, \id)}[\sigma\< (\id + \bDelta)^{1/2}\bw_i, \bx \>]\} / 2\\
=& \E_{G \sim \normal ( 0 ,1)} [ \sigma ( \| (\id - \bDelta)^{1/2} \bw_i \|_2 G) - \sigma ( \| (\id +  \bDelta)^{1/2} \bw_i \|_2G)  ]/2 . 
\end{aligned}
\]
We define three interpolating variables:
\[
\begin{aligned}
V^{(1)}_i & = \lambda_2 \{  \| (\id - \bDelta)^{1/2} \bw_i \|_2 - \| (1 + \bDelta)^{1/2} \bw_i \|_2 \}/2, \\
V^{(2)}_i & =  - \lambda_2 \{  \tr(\bDelta \bw_i \bw_i^\sT) \}/2, \\
V^{(3)}_i & =  - \lambda_2 \tr(\bDelta \bGamma) / 2.
\end{aligned}
\]
We begin by bounding the difference between $\bV$ and $\bV^{(1)}$. For convenience, we will define $\Tilde \bw_i = (\id - \bDelta)^{1/2}\bw_i$. We have:
\begin{equation}
\begin{aligned}
& \E [\sigma(\| \Tilde \bw_i \|_2 G)  - \sigma (G) ] - \lambda_2  (\|\Tilde \bw_i \|_2 - 1) \\
=& \E \Big[ \frac{\sigma(\| \Tilde \bw_i \|_2 G)  - \sigma (G) - (\|\Tilde \bw_i \|_2 - 1) G\sigma'(G) }{(\|\Tilde \bw_i \|_2 - 1)^2} \Big] (\|\Tilde \bw_i \|_2 - 1)^2.
\end{aligned}
\label{eq:diff_V_V1_RF_MG}
\end{equation}
Using dominated convergence theorem and arguments similar to those used to prove \eqref{eq:conv_lambda_2_dom}, one can check that
\begin{equation}
\lim_{t \to 1} \E \Big[ \frac{\sigma(t G)  - \sigma (G) - (t - 1) G\sigma'(G) }{(t - 1)^2} \Big] = ( \lambda_4 (\sigma) + \lambda_2 (\sigma) )/2.
\label{eq:seconde_dev_RF_MG}
\end{equation}
The same arguments as in the proofs of Lemma \ref{lem:RF_QF_kernel} and Lemma \ref{lem:QF_V} show
\begin{equation}
\begin{aligned}
\sup_{i \in [N] } | \| (\id - \bDelta)^{1/2}\bw_i \|_2 - 1 | & = o_{d,\P} (1),\\
\sum_{i=1}^N ( \| (\id - \bDelta)^{1/2}\bw_i \|_2 - 1 )^2 & = O_{d,\P} (1).
\end{aligned}
\label{eq:bound_wi_RF_MG}
\end{equation}
Combining \eqref{eq:seconde_dev_RF_MG} with \eqref{eq:bound_wi_RF_MG} in \eqref{eq:diff_V_V1_RF_MG}, we get:
\[
\begin{aligned}
&\sum_{i=1}^N \Big( \E [\sigma(\| \Tilde \bw_i \|_2 G)  - \sigma (G) ] - \lambda_2  (\|\Tilde \bw_i \|_2 - 1) \Big)^2 \\
= &\sum_{i=1}^N \Big(  \frac{\E [\sigma(\| \Tilde \bw_i \|_2 G)  - \sigma (G) ] - \lambda_2  (\|\Tilde \bw_i \|_2 - 1) }{(\|\Tilde \bw_i \|_2 - 1)^2}       \Big) (\|\Tilde \bw_i \|_2 - 1)^4 \\
 = & O_{d,\P} (1) \cdot \Big( \sup_{i \in [N] } | \| (\id - \bDelta)^{1/2}\bw_i \|_2 - 1 |^2 \Big) \sum_{i=1}^N ( \| (\id - \bDelta)^{1/2}\bw_i \|_2 - 1 )^2  = o_{d,\P} (1).
\end{aligned}
\]
Bounding similarly the term depending on $(\id + \bDelta)^{1/2}\bw_i$ in $V_i^{(1)}$, we get
\begin{equation}
    \label{eq:V1_V_MG}
    \| \bV - \bV^{(1)} \|_2 = o_{d,\P} (1).
\end{equation}
Now, consider the difference between $\bV^{(1)}$ and $\bV^{(2)}$. We use the fact for $x$ on a neighborhood of $0$, there exists $c$ such that
\[
| \sqrt{1-x} -  \sqrt{1 + x} + x  | \leq c | x |^3.
\]
Hence, with high probability
\[
\begin{aligned}
    | \| (\id - \bDelta)^{1/2} \bw_i \|_2 - \| (1 + \bDelta)^{1/2} \bw_i \|_2 + \< \bw_i , \bDelta \bw_i \> | \leq c\frac{|\< \bw_i , \bDelta \bw_i \>|^3 }{ \| \bw_i \|_2^2}.
\end{aligned}
\]
Furthermore, we have:
\[
\begin{aligned}
& \E_{\bw_i \sim \normal (\bzero, \bGamma) } \Big[ (\< \bw_i , \bDelta \bw_i \>)^6 / \| \bw_i \|_2^4 \Big]  \leq \| \bDelta \|_{\op}^2 \E [ (\< \bw_i , \bDelta \bw_i \>)^4 ] \\
& \leq C\| \bDelta \|_{\op}^2 ( \tr [ \bGamma^{1/2} \bDelta \bGamma^{1/2} ]^4 + \| \bGamma^{1/2} \bDelta \bGamma^{1/2} \|_F^4) = o_{d} (d^{-1}),
\end{aligned}
\]
where the last equality is due to assumptions {\bf M2} and {\bf B2}. We conclude that
\begin{equation}
    \label{eq:V2_V1_MG}
    \| \bV^{(1)} - \bV^{(2)} \|_2 = o_{d,\P} (1).
\end{equation}
For the last comparison between $\bV^{(2)}$ and $\bV^{(3)}$, we take the expectation:
\[
\begin{aligned}
\E_{\bw_i \sim \normal (\bzero , \bGamma ) } [( \< \bw_i , \bDelta \bw_i \> - \tr ( \bGamma \bDelta ) )^2] =& \E_{\bg \sim \normal (\bzero , \id ) } [(\< \bg \bg^\sT , \bGamma^{1/2} \bDelta \bGamma^{1/2} \> - \tr( \bGamma \bDelta) )^2] \\
=& 2 \| \bGamma^{1/2} \bDelta \bGamma^{1/2} \|^2_F \\
\leq & 2 \| \bGamma \|^2_{\text{op}} \| \bDelta \|^2_F=  O_{d} (d^{-2}).
\end{aligned}
\]
We get
\begin{equation}
    \label{eq:V3_V2_MG}
    \| \bV^{(3)} - \bV^{(2)} \|_2 = o_{d,\P} (1).
\end{equation}
Combining the above three bounds \eqref{eq:V1_V_MG}, \eqref{eq:V2_V1_MG} and \eqref{eq:V3_V2_MG} yields the desired result.
\end{proof}

\subsubsection{Proof of Theorem \ref{thm:RFMixture}}

By Lemma \ref{lem:GM_risk_representation}, the risk has a representation 
\[
R_{\RF,N}(f_*) = 1 - \bV^\sT \bU^{-1} \bV. 
\]
By Lemma \ref{lem:RF_GM_kernel}, we have 
\[
\| \bU - \bU_0 \|_{\op} = o_{d, \P}(1). 
\]
By Lemma \ref{lem:GM_V}, we have
\[
\| \bV - \tau \ones / \sqrt{d} \|_2 = o_{d, \P}(1),
\]
where
\[
\tau = - \sqrt{d} \cdot \lambda_2 \tr(\bDelta \bGamma) / 2. 
\]
Hence, we have 
\[
\vert \bV^\sT \bU^{-1} \bV - \tau^2 \ones^\sT \bU_{0}^{-1} \ones / d \vert = o_{d, \P}(1). 
\]
Proposition \ref{prop:one_U_one_expression} gives the expression
\[
\ones^\sT \bU_0^{-1} \ones / d = \psi / (1 + \kappa \psi) + o_{d, \P}(1), 
\]
where 
\[
\kappa = d \cdot \lambda_2^2 [\tr(\bGamma^2)/2 + \tr(\bDelta \bGamma)^2 / 4]. 
\]
Hence we have 
\[
\bV^\sT \bU \bV = \tau^2 \psi / (1 + \kappa \psi) + o_{d, \P}(1). 
\]
This proves the theorem.

\subsection{Neural Tangent model: proof of Theorem \ref{thm:NTK_MG}}

Recall the definition (note $R_{\NT, N}(\P)$ is a function of $\bW$)
\[
R_{\NT, N}(\P) = \arg\min_{\hf\in\cF_{\NT,N}(\bW)}\E \big\{(y -\hf(\bx))^2\big\},
\]
where
\[
\cF_{\NT, N}(\bW) = \Big\{ f_N( \bx) = c + \sum_{i=1}^N  \sigma'(\< \bw_i, \bx\>)\< \ba_i, \bx\>: c \in \R, \ba_i \in \R^d, i \in [N] \Big\}.
\]

\subsubsection{A representation lemma}

\begin{lemma}\label{eqn:representatation_risk_Gamma}
Assume conditions {\bf M1} and {\bf M2} hold. Consider the function
\begin{equation}
\hat f ( \bx ; \bGamma, a , c ) = a \< \bGamma , \bx \bx^\sT \> + c.
\label{eq:model_quad_func}
\end{equation}
Define the risk function optimized over $a,c$ while $\bGamma$ is fixed
\begin{equation}
\begin{aligned}
L ( \bGamma) = & \inf_{a, c}\E_{\bx,y} [ (y - \hat f ( \bx ; \bGamma, a ,c ))^2].
\end{aligned}
\end{equation}
Then we have 
\begin{equation}
\sup_{\bGamma \succeq 0} \Big\vert L (\bGamma)  - \frac{2}{2+ \< \bGamma , \bDelta \>^2/\| \bSigma^{1/2} \bGamma \bSigma^{1/2} \|_F^2}  \Big\vert=  o_d(1).
\label{eq:simplified_risk_MG}
\end{equation}
\end{lemma}

\begin{proof}[Proof of Lemma \ref{eqn:representatation_risk_Gamma}. ]
Note we have
\begin{equation*}
\begin{aligned}
L ( \bGamma, a , c) \equiv & \E_{\bx,y} [ (y - \hat f ( \bx ; \bGamma, a ,c ))^2  ] \\
= & 1 + c^2 +2ac \< \bGamma , \bSigma \> + 2 a \< \bGamma,\bDelta \> \\
& + a^2 [ \< \bGamma , \bSigma \>^2 + 2 \tr ( \bSigma \bGamma \bSigma \bGamma ) +  \< \bGamma , \bDelta \>^2 + 2 \tr ( \bDelta \bGamma \bDelta \bGamma )].
\end{aligned}
\label{eq:risk_quad_MG}
\end{equation*}
Minimizing successively over $c$ and $a$, we get the following formula:
\[
L (\bGamma) \equiv \min_{c,a \in \R} L ( \bGamma, a , c) =\frac{2}{2+ \< \bGamma , \bDelta \>^2 / [ \tr ( \bGamma \bSigma \bGamma \bSigma ) + \tr ( \bGamma \bDelta \bGamma \bDelta)]}.
\]
By Assumptions {\bf M1} and {\bf M2}, we have $\bSigma \succeq c \id_d$ and $\| \bDelta \|_{\op} \leq C/ \sqrt{d}$ for some constants $c$ and $C$. We get 
\[
\frac{\tr ( \bGamma \bDelta \bGamma \bDelta)}{\tr ( \bGamma \bSigma \bGamma \bSigma) } \leq \frac{C^2}{dc^2}.
\]
We deduce that
\begin{equation*}
\sup_{\bGamma \succeq 0} \Big\vert L (\bGamma)  - \frac{2}{2+ \< \bGamma , \bDelta \>^2/\| \bSigma^{1/2} \bGamma \bSigma^{1/2} \|_F^2}  \Big\vert \leq \Big\vert \frac{1}{1+ C^2/(dc^2)} - 1 \Big\vert   = o_d(1).
\end{equation*}
\end{proof}

\subsubsection{Proof of Theorem \ref{thm:NTK_MG}}

We consider the re-scaled matrices $\Tilde \bGamma = \bSigma^{1/2} \bGamma \bSigma^{1/2}$ and $\Tilde \bDelta = \bSigma^{-1/2} \bDelta \bSigma^{-1/2}$. We consider the $\NT$ model with a squared non-linearity:
\[
\hat f ( \bW , \bA ) = 2 \sum_{i=1}^N \< \bw_i, \bx \> \< \ba_i , \bx \> + c = 2 \< \bW \bA^\sT , \bx \bx^\sT \> + c. 
\]
with $\bW = [\bw_1 , \ldots , \bw_N] \in \R^{d \times N}$ and $\bA = [ \ba_1 , \ldots , \ba_N ] \in \R^{d \times N}$. For $\bw_i \sim \normal ( \bzero , \bSigma ) $, we have with probability one $\text{rank} (\bW ) =\min (d , N) \equiv r$. We consider $\bW = \bP_1 \bS \bV^\sT$ the SVD decomposition of $\bW$, with $\bP_1 \in \R^{d \times r}$, $\bS \in \R^{r \times r}$ and $\bV \in \R^{N \times r}$. Define $\bG =  \bS \bV^\sT \bA \in \R^{r \times d}$, we obtain almost surely that the minimum over $\bA$ is the same as the minimum over $\bG$. From Lemma \ref{eqn:representatation_risk_Gamma}, we deduce that almost surely
\begin{equation}
R_{\NT,N} (\P_{\bSigma,\bDelta}) = \min_{\bG \in \R^{d\times d}} \Bigg\{\frac{2}{2+\tr [ (\bP_1 \bG + \bG^\sT \bP_1^\sT ) \bDelta]^2/\|  \bP_1 \bG + \bG^\sT \bP_1^\sT  \|^2_{F}} \Bigg\} + o_d (1) 
\label{eq:simplified_G_risk}
\end{equation}

\noindent
{\bf Case $N / d \to \rho \ge 1$. } In the case $N \geq d$, we can take $\bG = \bP_1^\sT \Tilde \bG / 2$ and we get almost surely over $\bW \in \R^{d \times N}$
\[
R_{\NT,N} (\P_{\bSigma,\bDelta}) = \min_{\bG \in \R^{d\times d}} \Bigg\{\frac{2}{2+\< \bG , \bDelta \>^2/ \| \bG \|_F^2 } \Bigg\} + o_d (1) = \frac{2}{2 + \| \bDelta \|^2_F}+ o_d (1),
\]
where the minimizer $\bG = \bDelta$ is obtained by Cauchy-Schwarz inequality.

\noindent
{\bf Case $N / d \to \rho < 1$. } Consider now the case when $N < d$. From \eqref{eq:simplified_G_risk}, the optimal $\bG$ is the one maximizing
\[
\max_{\bG \in \R^{N \times d}}  \frac{\tr [ (\bP_1 \bG + \bG^\sT \bP_1^\sT ) \bDelta]^2 }{\|  \bP_1 \bG + \bG^\sT \bP_1^\sT  \|^2_{F}},
\]
which we rewrite as the following convex problem
\begin{equation}
\max_{\bG \in \R^{N \times d}} \, \tr [ \bP_1 \bG \bDelta ], \qquad \text{s.t.} \quad \| \bP_1 \bG + \bG^\sT \bP_1^\sT \|_F^2 \leq 1.
\label{eq:conv_program}
\end{equation}
We define $\bP_2 \in \R^{d \times (d - N)}$ the completion of $\bP_1 $ to a full basis $\bP = [ \bP_1 , \bP_2] \in \R^{d \times d}$, and denote $\bG_1 = \bG \bP_1 \in \R^{N \times N}$ and $\bG_2 = \bG \bP_2 \in \R^{N \times (d-N)}$. We can form the Lagrangian of problem \eqref{eq:conv_program}:
\[
\mathcal{L} (\bG , \lambda) = \tr ( \bP_1 \bG \bDelta ) + \lambda (1 - \| \bP_1 \bG + \bG^\sT \bP_1^\sT \|_F^2).
\]
The stationary condition implies:
\[
\nabla_\bG \mathcal{L} (\bG , \lambda) = \bP_1^\sT \bDelta - 4\lambda ( \bP_1^\sT  \bG^\sT \bP_1^\sT + \bP_1^\sT \bP_1 \bG ) = 0,
\]
which yields, using $\bP_1^\sT \bP_1 = \id_N$, 
\begin{equation}
\bDelta_{12} = 4 \lambda \bG_2, \qquad \bDelta_{11}  = 4 \lambda ( \bG_1 + \bG_1^\sT ),
\label{eq:stationary_NT_MG}
\end{equation}
where $\bDelta_{ij} = \bP_i^\sT \bDelta \bP_j $ for $i,j = 1,2$. The constraint reads in the $\bP$ basis
\begin{equation}
\| \bP_1 \bG + \bG^\sT \bP_1^\sT \|_F^2 = \| \bG_1 + \bG_1^\sT \|_F^2 + 2 \| \bG_2 \|_F^2 = 1.
\label{eq:constraint_NT_MG}
\end{equation}
Substituting \eqref{eq:stationary_NT_MG} in \eqref{eq:constraint_NT_MG} yields:
\begin{equation}
4 \lambda = \sqrt{\|\bDelta_{11} \|_F^2 + 2 \| \bDelta_{12} \|_F^2 }. 
\label{eq:lambda_NT_MG}
\end{equation}
Considering the (unique) symmetric optimizer $\bG_1$ and substituting \eqref{eq:lambda_NT_MG} in \eqref{eq:stationary_NT_MG}, we get the minimizer
\begin{equation}
\begin{aligned}
\bG^*_1 & = \frac{1}{8 \lambda} \bDelta_{11} = \frac{1}{2  \sqrt{\|\bDelta_{11} \|_F^2 + 2 \| \bDelta_{12} \|_F^2 } } \bDelta_{11}, \\
\bG^*_2 & = \frac{1}{4 \lambda} \bDelta_{12} = \frac{1}{  \sqrt{\|\bDelta_{11} \|_F^2 + 2 \| \bDelta_{12} \|_F^2 } } \bDelta_{12}.
\end{aligned}
\label{eq:minimizer_NT_MG}
\end{equation}
Let's consider the objective function:
\begin{align}
\tr ( \bP_1 \bG^* \bDelta ) =& \tr ( \bG_1^* \bDelta_{11} + \bG^*_2 \bDelta_{21}) \nonumber \\
= & \frac{1}{2  \sqrt{\|\bDelta_{11} \|_F^2 + 2 \| \bDelta_{12} \|_F^2 }} \tr ( \bDelta_{11}^2 + 2 \bDelta_{12} \bDelta_{21})\nonumber  \\
= & \frac{1}{2}  \sqrt{\|\bDelta_{11} \|_F^2 + 2 \| \bDelta_{12} \|_F^2 } \nonumber \\
= &  \frac{1}{2} \sqrt{\| \bDelta \|_F^2 - \| \bDelta_{22} \|^2_F}. \label{eq:tr_min_NTK_MG}
\end{align}
Substituting \eqref{eq:tr_min_NTK_MG} in \eqref{eq:simplified_G_risk}, we then obtain
\begin{equation}
R_{\NT,N} (\P_{\bSigma,\bDelta})  = \frac{2}{2 + \| \bDelta \|^2_F - \| \bDelta_{22} \|^2_F}+ o_d (1),
\label{eq:obj_NTK_delta22}
\end{equation}
where $\bDelta_{22} =\bP_{\bW^\perp} \bDelta \bP_{\bW^\perp}  $ with $\bP_{\bW^\perp} = \id_d - \bW ( \bW^\sT \bW)^{-1} \bW^\sT$ is the random projection along the orthogonal subspace to the columns of $\bW$.
From Theorem \ref{thm:QF_NTK}, we know that 
\begin{equation}\label{eqn:E_Delta_22}
\E [ \| \bDelta_{22} \|_F^2 ] = \| \bDelta \|_F^2 \Big[ (1-\rho)_+^2 \Big(1- \frac{\tr (\bDelta)^2 }{d \| \bDelta \|_F^2}\Big) + (1 - \rho)_+ \frac{\tr (\bDelta)^2 }{d \| \bDelta \|_F^2}  + o_d (1)\Big].
\end{equation}
Let $\mathbb{W}^N_d$ be the Stiefel manifold, i.e. the collection of all the sets of $N$ orthonormal vectors in $\R^d$ endowed with the Frobenius distance. In matrix representation, we have
\[
\mathbb{W}^N_d = \{ \bP \in \R^{d \times N}: \bP^\sT \bP = \id_N \}. 
\]

By Theorem 2.4 in \cite{ledoux2001concentration}, the volume measure on $\mathbb{W}^N_d$ has normal concentration. In particular, denote by $F : \mathbb{W}^N_d \mapsto \R $, the function $F(\bP) = \| \bP^\sT \bDelta \bP \|_F^2$. We upper bound the gradient of $F$:
\[
\|\nabla F(\bP) \|_F = 4 \| \bDelta \bP \bP^\sT \bDelta \bP \|_F \leq 4 \| \bDelta \bP \bP^\sT \|_{\text{op}}  \| \bDelta \bP  \|_{F} \leq \| \bDelta  \|_{\text{op}}  \| \bDelta  \|_{F} \leq C,
\]
by assumption {\bf M2} on $\bDelta$. We deduce that there exists a constant $c$ (that depends on $\rho$ and $C$) such that:
\[
\P ( | F(\bP) - \E [ F(\bP) ] | > t ) \leq e^{-c d t^2}.
\]
Therefore, we have 
\begin{equation}
\P ( | \| \bDelta_{22} \|_F^2 - E [ \| \bDelta_{22} \|_F^2 ]  | > t ) \leq e^{-c d t^2}.
\label{eq:high_proba_delta22}
\end{equation}
Using \eqref{eq:high_proba_delta22} and \eqref{eq:obj_NTK_delta22}, we deduce the final high probability formula for the risk of the \NT \,model:
\[
R_{\NT,N} (\P_{\bSigma,\bDelta})  = \frac{2}{2 + \| \bDelta \|^2_F - \E [ \| \bDelta_{22} \|_F^2 ] }+ o_{d,\P} (1). 
\]
Substituting $\E[\| \bDelta_{22} \|_F^2]$ by its expression (\ref{eqn:E_Delta_22}) concludes the proof.

\subsection{Neural Network model: proof of Theorem \ref{thm:NN_MG}}

Recall the definition 
\[
R_{\NN, N}(\P) = \arg\min_{\hf\in\cF_{\NN,N}(\bW)}\E\big\{(y-\hf(\bx))^2\big\},
\]
where we consider the function class of two-layers neural networks (with $N$ neurons) with quadratic activation function and general offset and coefficients
\[
\cF_{\NN, N}(\bW) = \Big\{ f_N( \bx) = c + \sum_{i=1}^N  a_i (\< \bw_i, \bx\>)^2: \; c, a_i \in \R, i \in [N] \Big\}.\\
\]
We define the risk function for a given set of parameters as
\[
L(\bW, \ba, c) = \E_{\bx,y}[(y - \hat f(\bx; \bW, \ba, c))^2]. 
\]
The risk is optimized over $(a_i,\bw_i)_{i\le N}$ and $c$.

\begin{proof}[Proof of Theorem \ref{thm:NN_MG}]
Without loss of generality, we assume $\bSigma = \id_d$ (it suffices to consider the re-scaled matrices $\Tilde \bGamma = \bSigma^{1/2} \bGamma \bSigma^{1/2}$ and $\Tilde \bDelta = \bSigma^{-1/2} \bDelta \bSigma^{-1/2}$). We rewrite the neural network function in a compact form:
\[
\hat f (\bx ; \bW ,\ba,c) =  \sum_{i=1}^N a_i \< \bw_i , \bx \>^2 +c =  \< \bW \bA \bW^\sT , \bx \bx^\sT \> + c,
\]
where $\bA = \diag (\ba)$. Define $\bGamma = \bW \bA \bW^\sT$ and using Eq.~\eqref{eq:simplified_risk_MG}  in Lemma \ref{eqn:representatation_risk_Gamma}, the minimizer $\bGamma^*$ is the solution of
\[
\max_{\bGamma \in  \mathcal{S} ( \R^{d \times d})} \, \frac{\< \bGamma ,  \bDelta \>^2 }{\| \bGamma \|_F^2}, \qquad \text{s.t.} \quad \text{rank}(\bGamma) \leq \min(N,d) \equiv r.
\]
where $ \mathcal{S} ( \R^{d \times d})$ is the set of symmetric matrices in $\R^{d \times d}$.

Let us denote the eigendecomposition of $\bGamma$ by $\bGamma = \bU \bS \bU^\sT$ with $\bU \in \R^{d \times r}$ and $\bS = \diag (\bs) \in \R^{r \times r}$. We have by Cauchy-Schwartz inequality
\[
\frac{\< \bGamma ,  \bDelta \>^2 }{\| \bGamma \|^2_F}= \frac{\tr ( \bS \bU^\sT \bDelta \bU )^2}{ \| \bS \|_F^2 } \leq \| \diag( \bU^\sT \bDelta \bU ) \|_2^2,
\]
with equality if and only if $\bS_* = \text{{\rm ddiag}} ( \bU^\sT \bDelta \bU )$ where $\text{{\rm ddiag}}  ( \bU^\sT \bDelta \bU )$ is the vector of the diagonal elements of $\bU^\sT \bDelta \bU $. Denoting $\mathcal{D} (\R^{d \times d})$ the set of diagonal matrices in $\R^{d \times d}$, we get 
\[
\max_{\bS \in \mathcal{D} (\R^{d \times d})} \frac{\< \bU \bS \bU^\sT ,  \bDelta \>^2 }{\| \bU \bS \bU^\sT \|^2_F} =  \frac{\< \bS_* ,  \bU^\sT \bDelta \bU \>^2 }{\| \bS_*  \|_F^2} = \frac{\|\bS_* \|^4_F }{\| \bS_*\|^2_F} = \|\bS_* \|^2_F.
\]
Hence, the problem reduces to finding $\bU \in \R^{d \times r}$ with orthonormal columns which maximizes $\|  \text{{\rm ddiag}} ( \bU^\sT \bDelta \bU ) \|_F^2$. The maximizer is easily found as the eigendirections corresponding to the $r$ largest singular values. We conclude that at the optimum  
\[
\frac{\< \bGamma_* ,  \bDelta \>^2 }{\| \bGamma_* \|_F^2} = \sum_{i= 1}^{r} \lambda_{i}^2,
\]
where the $\lambda_{i}$'s are the singular values of $\bDelta$ in descending order. Plugging this expression in Eq.~\eqref{eq:simplified_risk_MG} concludes the proof.
\end{proof}

\end{document}